\documentclass{article}
\usepackage{fullpage}

\usepackage{amsmath,amsfonts,bm}

\def\eqref#1{equation~\ref{#1}}
\def\1{\bm{1}}

\DeclareMathAlphabet{\mathsfit}{\encodingdefault}{\sfdefault}{m}{sl}
\SetMathAlphabet{\mathsfit}{bold}{\encodingdefault}{\sfdefault}{bx}{n}

\newcommand{\E}{\mathbb{E}}

\newcommand{\R}{\mathbb{R}}

\newcommand{\Var}{\mathrm{Var}}

\usepackage{hyperref}
\usepackage{url}

\usepackage{verbatim}
\usepackage{amsmath, amsthm, amssymb}\allowdisplaybreaks
\usepackage{ifthen}
\usepackage{color, xcolor}

\usepackage{hyperref}
\hypersetup{colorlinks=true, linkcolor=blue, citecolor=blue, anchorcolor=blue, urlcolor=blue}  % link color
\usepackage{graphicx}
\usepackage{tikz}
\usepackage{subfig}

\usepackage{natbib}
\setcitestyle{authoryear}

\newtheorem{theorem}{Theorem}
\newtheorem{lemma}{Lemma}

\newtheorem{proposition}{Proposition}

\newtheorem{definition}{Definition}

\newcommand{\I}{\mathbb{I}}
\newcommand{\N}{\mathbb{N}}

\newcommand{\bx}{\boldsymbol{x}}

\newcommand{\cN}{\mathcal{N}}
\newcommand{\cR}{\mathcal{R}}

\newcommand{\cP}{{\mathcal{P}}}

\usepackage{algorithm}
\usepackage{algorithmicx}
\usepackage{algpseudocode}
\algrenewcommand\algorithmicrequire{\textbf{Input:}}
\algrenewcommand\algorithmicensure{\textbf{Output:}}

\usepackage{thm-restate}

\newcommand{\compilehidecomments}{false}%HIDE comments
\ifthenelse{\equal{\compilehidecomments}{true}}{%
    \newcommand{\rong}[1]{}
    \newcommand{\haoyu}[1]{}
    \newcommand{\runzhe}[1]{}
}{
    \newcommand{\rong}[1]{{\color{blue!60!black}  [\text{Rong:} #1]}}
    \newcommand{\haoyu}[1]{{\color{brown!60!black} [\text{Haoyu:} #1]}}
    \newcommand{\runzhe}[1]{{\color{red!60!black} [\text{Runzhe:} #1]}}
}

\newcommand{\compilefullversion}{true}%SHOW short version
\ifthenelse{\equal{\compilefullversion}{false}}{%
    \newcommand{\OnlyInFull}[1]{}
    \newcommand{\OnlyInShort}[1]{#1}
}{%
    \newcommand{\OnlyInFull}[1]{#1}%
    \newcommand{\OnlyInShort}[1]{}%
}%

\title{Mildly Overparametrized Neural Nets can Memorize Training Data Efficiently}
\date{}
\author{
Rong Ge\\
Duke University\\
\texttt{rongge@cs.duke.edu}
\and 
Runzhe Wang\thanks{This work is completed when Runzhe Wang and Haoyu Zhao were visiting students at Duke University.}\\
IIIS, Tsinghua University\\
\texttt{wrz16@mails.tsinghua.edu.cn}
\and 
Haoyu Zhao\footnotemark[\value{footnote}]\\
IIIS, Tsinghua University\\
\texttt{zhaohy16@mails.tsinghua.edu.cn}\\
}

\begin{document}
\maketitle

\begin{abstract}
    It has been observed \citep{zhang2016understanding} that deep neural networks can memorize: they achieve 100\% accuracy on training data. Recent theoretical results explained such behavior in highly overparametrized regimes, where the number of neurons in each layer is larger than the number of training samples. In this paper, we show that neural networks can be trained to memorize training data perfectly in a mildly overparametrized regime, where the number of parameters is just a constant factor more than the number of training samples, and the number of neurons is much smaller.
\end{abstract}

\section{Introduction}

In deep learning, highly non-convex objectives are optimized by simple algorithms such as stochastic gradient descent. However, it was observed that neural networks are able to fit the training data perfectly, even when the data/labels are randomly corrupted\citep{zhang2016understanding}. Recently, a series of work (\citet{du2019gradient, allen2019convergence, chizat2018global, jacot2018neural}, see more references in Section~\ref{sec:related}) developed a theory of neural tangent kernels (NTK) that explains the success of training neural networks through overparametrization. Several results showed that if the number of neurons at each layer is much larger than the number of training samples, networks of different architectures (multilayer/recurrent) can all fit the training data perfectly.

However, if one considers the number of parameters required for the current theoretical analysis, these networks are highly overparametrized. Consider fully connected networks for example. If a two-layer network has a hidden layer with $r$ neurons, the number of parameters is at least $rd$ where $d$ is the dimension of the input. For deeper networks, if it has two consecutive hidden layers of size $r$, then the number of parameters is at least $r^2$. All of the existing works require the number of neurons $r$ per-layer to be at least the number of training samples $n$ (in fact, most of them require $r$ to be a polynomial of $n$). In these cases, the number of parameters can be at least $nd$ or even $n^2$ for deeper networks, which is much larger than the number of training samples $n$. Therefore, a natural question is whether neural networks can fit the training data in the mildly overparametrized regime - where the number of (trainable) parameters is only a constant factor larger than the number of training data. To achieve this, one would want to use a small number of neurons in each layer - $n/d$ for a two-layer network and $\sqrt{n}$ for a three-layer network. \citet{yun2018small} showed such networks have enough capacity to memorize any training data. In this paper we show with polynomial activation functions, simple optimization algorithms are guaranteed to find a solution that memorizes training data.

%Of course, there are trivial ways to construct neural networks that can work in the mildly overparametrized regime. For example, consider a two-layer neural network where the hidden layer has $h$ neurons and $h\ge n$. Suppose the first layer has fixed random weights, and the second layer has $h$ trainable parameters. This is equivalent to fitting the data using random features, which is a convex optimization problem. If the random features in the hidden layer are linearly independent for all input, it is easy to see that the network can fit any output. One can avoid random features by requiring the number of neurons in each layer to be smaller than the number of training samples - in this case, the algorithm has to find reasonable weights in order to fit the data perfectly. 

\subsection{Our Results}

In this paper, we give network architectures (with polynomial activations) such that every hidden layer has size much smaller than the number of training samples $n$, the total number of parameters is linear in $n$, and simple optimization algorithms on such neural networks can fit any training data.

We first give a warm-up result that works when the number of training samples is roughly $d^2$ (where $d$ is the input dimension).

\begin{theorem}[Informal]\label{thm:main1:informal} Suppose there are $n \le {d+1 \choose 2}$ training samples in general position, there exists a two-layer neural network with quadratic activations, such that the number of neurons in the hidden layer is $2d+2$, the total number of parameters is $O(d^2)$, and perturbed gradient descent can fit the network to any output.
\end{theorem}

%\rong{Please fix constants if they are wrong}

Here ``in general position'' will be formalized later as a deterministic condition that is true with probability 1 for random inputs, see Theorem~\ref{thm:main-theorem-twolayer} for details.

In this case, the number of hidden neurons is only roughly the square root of the number of training samples, so the weights for these neurons need to be trained carefully in order to fit the data. Our analysis relies on an analysis of optimization landscape - we show that every local minimum for such neural network must also be globally optimal (and has 0 training error). As a result, the algorithm can converge from an arbitrary initialization.

Of course, the result above is limited as the number of training samples cannot be larger than $O(d^2)$. We can extend the result to handle a larger number of training samples:

\begin{theorem}[Informal]\label{thm:main2:informal} Suppose number of training samples $n \le d^p$ for some constant $p$, if the training samples are in general position there exists a three-layer neural network with polynomial activations, such that the number of neurons $r$ in each layer is $O_p(\sqrt{n})$, and perturbed gradient descent on the middle layer can fit the network to any output.
\end{theorem}

Here $O_p$ considers $p$ as a constants and hides constant factors that only depend on $p$. We consider ``in general position'' in the smoothed analysis framework\citep{spielman2004smoothed} - given arbitrary inputs $x_1,x_2,...,x_n \in \R^d$, fix a perturbation radius $\sqrt{v}$, the actual inputs is $\bar{x}_j = x_j+\tilde{x}_j$ where $\tilde{x}_j\sim N(0, vI)$. The guarantee of training algorithm will depend inverse polynomially on the perturbation $v$ (note that the architecture -  in particular the number of neurons - is independent of $v$). The formal result is given in Theorem~\ref{thm:main-theorem-random-feature-with-perturbation} in Section~\ref{sec:proof-sketch-random-feature}. Later we also give a deterministic condition for the inputs, and prove a slightly weaker result (see Theorem~\ref{thm:deterministic}).

\subsection{Related Works}
\label{sec:related}
\paragraph{Neural Tangent Kernel} Many results in the framework of neural tangent kernel show that networks with different architecture can all memorize the training data, including two-layer \citep{du2019gradient}, multi-layer\citep{du2018gradient2, allen2019convergence, zou2019improved}, recurrent neural network\citep{allen2018convergence}. However, all of these works require the number of neurons in each layer to be at least quadratic in the number of training samples. \citet{oymak2019towards} improved the number of neurons required for two-layer networks, but their bound is still larger than the number of training samples. There are also more works for NTK on generalization guarantees (e.g., \citet{allen2018learning}), fine-grained analysis for specific inputs\citep{arora2019fine} and empirical performances\citep{arora2019exact}, but they are not directly related to our results. 
\paragraph{Representation Power of Neural Networks} For standard neural networks with ReLU activations, \cite{yun2018small} showed that networks of similar size as Theorem~\ref{thm:main2:informal} can memorize any training data. Their construction is delicate and it is not clear whether gradient descent can find such a memorizing network efficiently.

\paragraph{Matrix Factorizations} Since the activation function for our two-layer net is quadratic, training of the network is very similar to matrix factorization problem. Many existing works analyzed the optimization landscape and implicit bias for problems related to matrix factorization in various settings\citep{bhojanapalli2016global, ge2016matrix, ge2017no, park2016non, gunasekar2017implicit, li2018algorithmic, arora2019implicit}. In this line of work, \citet{du2018power} is the most similar to our two-layer result, where they showed how gradient descent can learn a two-layer neural network that represents any positive semidefinite matrix. However positive definite matrices cannot be used to memorize arbitrary data, and our two-layer network can represent an arbitrary matrix.

\paragraph{Interpolating Methods} Of course, simply memorizing the data may not be useful in machine learning. However, recently several works\citep{belkin2018overfitting, belkin2019does, liang2019risk, mei2019generalization}  showed that learning regimes that interpolate/memorize data can also have generalization guarantees. Proving generalization for our architectures is an interesting open problem.

\section{Preliminaries}
%\haoyu{The reference for the informal statement of the theorem or the formal statement?}
In this section, we introduce notations, the two neural network architectures used for Theorem~\ref{thm:main1:informal} and \ref{thm:main2:informal}, and the perturbed gradient descent algorithm.

\subsection{Notations}
%\rong{Introduce basic notations, $\otimes$, matrix norms, eigenvalues}
%\haoyu{Add the following intro for notations, please check.}

We use $[n]$ to denote the set $\{1,2,...,n\}$.
For a vector $x$, we use $\|x\|_2$ to denote its $\ell_2$ norm, and sometimes $\|x\|$ as a shorthand. 
%\rong{Do we use any other $\ell_p$ norms?}
%\haoyu{No, so I just change them into the 2 norm and add a shorthand notation.}
For a matrix $M$, we use $\|M\|_{F}$ to denote its Frobenius norm, $\|M\|$ to denote its spectral norm. We will also use $\lambda_i(M)$ and $\sigma_i(M)$ to denote the $i$-th largest eigenvalue and singular value of matrix $M$, and $\lambda_{\min}(M)$, $\sigma_{\min}(M)$ to denote the smallest eigenvalue/singular value. % and the smallest singular value of matrices.

For the results of three-layer networks, our activation is going to be $x^p$, where $p$ is considered as a small constant. We use $O_p()$, $\Omega_p()$ to hide factors that only depend on $p$.

%the spectral norm, and $||\cdot||_p$ for the $p$-norm. 
For vectors $x,y\in \R^d$, the tensor product is denoted by $(x\otimes x)\in \R^{d^2}$. We use $x^{\otimes p} \in \R^{d^p}$ as a shorthand for $p$-th power of $x$ in terms of tensor product. For two matrices $M,N\in \R^{d_1\times d_2}$, we use $M\otimes N \in \R^{d_1^2\times d_2^2}$ denote the Kronecker product of 2 matrices. %\haoyu{I change the size of matrices since the matrices are not guaranteed to be square.}

\subsection{Network Architectures}
\label{sec:prelim_architecture}
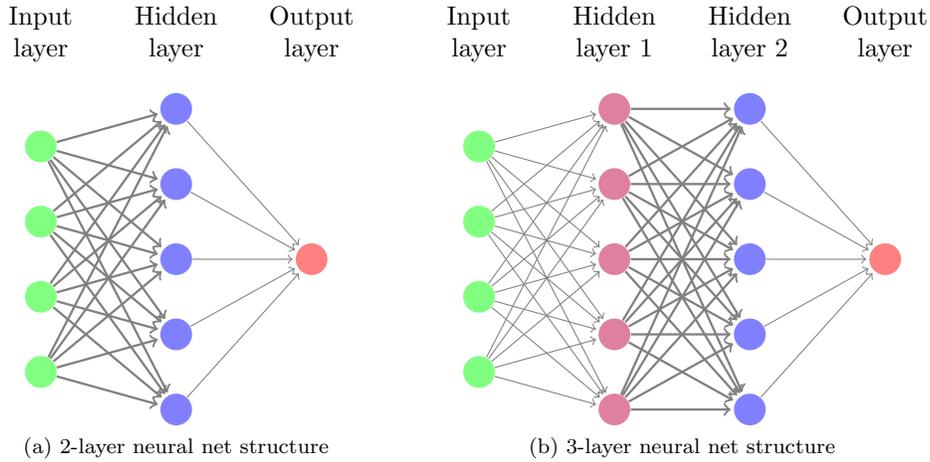
\begin{figure}[tb]
\def\layersep{1.8cm}
\def\tlayersep{3.6cm}
\centering
\subfloat[2-layer neural net structure]{
\begin{tikzpicture}[shorten >=1pt,->,draw=black!50, node distance=\layersep]
    \tikzstyle{every pin edge}=[<-,shorten <=1pt]
    \tikzstyle{neuron}=[circle,fill=black!25,minimum size=12pt,inner sep=0pt]
    \tikzstyle{input neuron}=[neuron, fill=green!50];
    \tikzstyle{output neuron}=[neuron, fill=red!50];
    \tikzstyle{hidden neuron}=[neuron, fill=blue!50];
    \tikzstyle{annot} = [text width=4em, text centered]

    % Draw the input layer nodes
    \foreach \name / \y in {1,...,4}
    % This is the same as writing \foreach \name / \y in {1/1,2/2,3/3,4/4}
        \node[input neuron] (I-\name) at (0,-\y) {};

    % Draw the hidden layer nodes
    \foreach \name / \y in {1,...,5}
        \path[yshift=0.5cm]
            node[hidden neuron] (H-\name) at (\layersep,-\y cm) {};

    % Draw the output layer node
    \node[output neuron, right of=H-3] (O) {};

    % Connect every node in the input layer with every node in the
    % hidden layer.
    \foreach \source in {1,...,4}
        \foreach \dest in {1,...,5}
            \path[line width=0.3mm] (I-\source) edge (H-\dest);

    % Connect every node in the hidden layer with the output layer
    \foreach \source in {1,...,5}
        \path (H-\source) edge (O);

    % Annotate the layers
    \node[annot,above of=H-1, node distance=1cm] (hl) {Hidden layer};
    \node[annot,left of=hl] {Input layer};
    \node[annot,right of=hl] {Output layer};
\end{tikzpicture}
}
\quad
\subfloat[3-layer neural net structure]{
\begin{tikzpicture}[shorten >=1pt,->,draw=black!50, node distance=\layersep]
    \tikzstyle{every pin edge}=[<-,shorten <=1pt]
    \tikzstyle{neuron}=[circle,fill=black!25,minimum size=12pt,inner sep=0pt]
    \tikzstyle{input neuron}=[neuron, fill=green!50];
    \tikzstyle{output neuron}=[neuron, fill=red!50];
    \tikzstyle{hidden neuron 2}=[neuron, fill=blue!50];
    \tikzstyle{hidden neuron 1}=[neuron, fill=purple!50];
    \tikzstyle{annot} = [text width=4em, text centered]

    % Draw the input layer nodes
    \foreach \name / \y in {1,...,4}
    % This is the same as writing \foreach \name / \y in {1/1,2/2,3/3,4/4}
        \node[input neuron] (I-\name) at (0,-\y) {};

    % Draw the hidden layer nodes
    \foreach \name / \y in {1,...,5}
        \path[yshift=0.5cm]
            node[hidden neuron 1] (H1-\name) at (\layersep,-\y cm) {};
    
    \foreach \name / \y in {1,...,5}
        \path[yshift=0.5cm]
            node[hidden neuron 2] (H2-\name) at (\tlayersep,-\y cm) {};

    % Draw the output layer node
    \node[output neuron,right of=H2-3] (O) {};

    % Connect every node in the input layer with every node in the
    % hidden layer.
    \foreach \source in {1,...,4}
        \foreach \dest in {1,...,5}
            \path (I-\source) edge (H1-\dest);
            
    \foreach \source in {1,...,5}
        \foreach \dest in {1,...,5}
            \path[line width=0.3mm] (H1-\source) edge (H2-\dest);

    % Connect every node in the hidden layer with the output layer
    \foreach \source in {1,...,5}
        \path (H2-\source) edge (O);

    % Annotate the layers
    \node[annot,above of=H1-1, node distance=1cm] (hl) {Hidden layer 1};
    \node[annot,above of=H2-1, node distance=1cm] (hr) {Hidden layer 2};
    \node[annot,left of=hl] {Input layer};
    \node[annot,right of=hr] {Output layer};
\end{tikzpicture}
}
\caption{Neural Network Architectures. The trained layer is in bold face. The activation function after the trained parameters is $x^2$(blue neurons). The activation function before the trained parameters is $x^p$(purple neurons).}
\label{fig:architecture}
\end{figure}

%\rong{The results are a bit too long here. I would just introduce the network architecture here (with pictures), and then give the actual theorems later in section 3/4. Also move the introduction to perturbed gradient descent a separate subsection.}

%\haoyu{Change the structure, please see the revise version and my other comments.}

%\haoyu{I just add 2 simple graphs, please have a check and give some feedback on what illustrations we need to add on the figures.}

%\rong{I think we should highlight the trained layer (maybe just in bold), and mention the activation somewhere}

%\haoyu{I change the picture and add some illustration in the caption.}

%\label{sec:model-results}
In this section, we introduce the neural net architectures we use. As we discussed, Theorem~\ref{thm:main1:informal} uses a two-layer network (see Figure~\ref{fig:architecture} (a)) and Theorem~\ref{thm:main2:informal} uses a three-layer network (see Figure~\ref{fig:architecture} (b)). % We consider 2 kinds of neural networks. The first kind is a 2-layer(1-hidden layer) neural network with quadratic activation, and the second is a 3-layer neural network with the first layer being a random feature layer. In both structures, we only train the parameters of the upper first layer. The first 2-layer model is simple, but to attain zero training loss there must be the restraint that $n = O(d^2)$, where $n$ is the number of samples and $d$ is the input dimension of each sample. The second 3-layer model can deal with the case where $n = O(d^{2p})$ with $p$ being a constant degree of the activation polynomial for the bottom layer.

\paragraph{Two-layer Neural Network} For the two-layer neural network, suppose the input samples $x$ are in $\R^d$, the hidden layer has $r$ hidden neurons (for simplicity, we assume $r$ is even, in Theorem~\ref{thm:main-theorem-twolayer} we will show that $r = 2d+2$ is enough). The activation function of the hidden layer is $\sigma(x) = x^2$.

We use $w_i\in\R^d$ to denote the input weight of hidden neuron $i$. These weight vectors are collected as a weight matrix $W = [w_1,w_2,\dots,w_r] \in \R^{d\times r}$. The output layer has only 1 neuron, and we use $a_i\in\R$ to denote the its input weight from hidden neuron $i$. There is no nonlinearity for the output layer. For simplicity, we fix the parameters $a_i,i\in [r]$ in the way that $a_i = 1$ for all $1\le i\le \frac{r}{2}$ and $a_i = -1$ for all $\frac{r}{2}+1 \le i \le r$. Given $x$ as the input, the output of the neural network is
\[y = \sum_{i=1}^r a_i(w_i^Tx)^2.\]

%Here, we consider a 2-layer (1-hidden layer) deep neural network. There are $n$ samples $x_1,\dots,x_n\in\R^d$ and their corresponding labels $y_1,\dots,y_n\in\R$. There are $r$ hidden neurons in the hidden layer, and for simplicity we assume that $r$ is an even number. The activation function of the hidden layer is $\sigma(x) = x^2$. As the network is fully connected, we use 
If the training samples are $\{(x_j,y_j)\}_{j\le n}$, we define the empirical risk of the neural network with parameters $W$ to be
\[f(W) = \frac{1}{4n}\sum_{j=1}^n\left(\sum_{i=1}^r a_i(w_i^Tx_j)^2 - y_j\right)^2.\]

\paragraph{Three-layer neural network} For Theorem~\ref{thm:main2:informal}, we use a more complicated, three-layer neural network. 
%In the previous section, we introduce a 2-layer(1-hidden layer) neural network model using quadratic activation. However, in order to get very small training error in polynomial time, we assume that the second power tensor products of each sample are linearly independent. Based on that assumption, the number of samples cannot exceed $O(d^2)$, where $d$ is the dimension of each sample. To resolve this problem, we consider a 3-layer neural network, where 
In this network, the first layer has a polynomial activation $\tau(x) = x^p$,
%\rong{Maybe we want to use a different letter for this activation}
and the next two layers are the same as the two-layer network. 
%\haoyu{Is $h(x)$ ok? I can change this in the whole paper.} \rong{Might use another greek letter like $\tau(x)$?} 
%We do not train the random feature layer and we only train the parameters in the middle layer. Next, we will give our formal definition of our model.

We use $R = [r_1,\dots, r_k]^T\in\mathbb R^{k\times d}$ 
%\rong{There were a lot of confusions with $K$ and $k$ here, I think we should stick with lowercase letters. Also this seems to be inconsistent with $W$ which is $d\times r$}
%\haoyu{I will change all of them from $K$ into $k$. However, in the following proofs, we might not decide how large should $k$ be so we use $K$ first.} 
%\haoyu{Actually, I think that we may need to change a lot of notations. First the number of neurons with quadratic activation, we use $r$ but $r_i$ is denoted as the random feature parameters. Second, we use $\sigma$ to denote the singular value, so it may not be good to use $\sigma$ to denote the activation function again. I suggest to use $m$ to denote the number of neurons instead of $r$, and maybe $h_q()$ to denote the quadratic activation and $h_r()$ to denote the $x^p$ activation.}
%\rong{I'm not sure what do you mean by ``we might not decide how large should $k$ be so we use $K$ first''. I think it's fine to use $k$ throughout and just say $k$ is a parameter that we will determine later. For the other parameters, yes $r$ probably should be replaced by $m$. I wouldn't worry about using $\sigma$ for both activation and singular values, because when you use it for singular values you always use either $\sigma_{min}$ or $\sigma_i$. Using $\sigma$ for activation is fairly standard so we can stick to that (and again maybe use $\tau$ for the first layer).}
%\haoyu{Ok, I will fix all the variables as quick as possible}
to denote the weight parameter of the first layer. The first hidden layer has $k$ neurons with activation $\tau(x) = x^p$ where $p$ is the parameter in Theorem~\ref{thm:main2:informal}. Given input $x$, the output of the first hidden layer is denoted as $z$, and satisfy $z_i = (r_i^T x_j)^p$.%\haoyu{I change the subscript from $k$ to $i$ since $k$ is used before. By the way, we do not mention the perturbation on samples and do we need to mention that in this paragraph, or in the later section introducing the main theorems?}\rong{Thanks. We don't need to worry about the perturbation here.} 
%The first layer output of $x_j$ is $z_j\in\mathbb R^k$, where $(z_j)_k = (r_k^T x_j)^p$. Then the next 2 layer is the same as the 2 layer network we introduced before, but now the input is $\{z_j,y_j\}_{j\le n}$ instead of $\{x_j,y_j\}_{j\le n}$. 
The second hidden layer has $r$ neurons (again we will later show $r = 2k+2$ is enough).
%\rong{Is this the correct number?}\haoyu{I find that we do not state that $r\ge 2d+2$ in the previous paragraph which discusses the 2-layer structure. Do we need to first say that we have $r$ neurons and than state that $r\ge 2k+2$ in our main theorem?} 
The weight matrix for second layer is denoted as $W = [w_1,\dots, w_{r}]\in \R^{k\times r}$ where each $w_i\in \R^k$ is the weight for a neuron in the second hidden layer. The activation for the second hidden layer is $\sigma(x) = x^2$. The third layer has weight $a$ and is initialized the same way as before, where $a_1 = a_2=\cdots = a_{r/2} = 1$, and $a_{r/2+1} = \cdots = a_{r} = -1$. The final output $y$ can be computed as
\[
y = \sum_{i=1}^{r} a_i (w_i^T z)^2.
\]

Given inputs $(x_1,y_1), ..., (x_n,y_n)$, suppose $z_i$ is the output of the first hidden layer for $x_i$, the empirical loss is defined as:
\[f(W) = \frac{1}{4n}\sum_{j=1}^n\left(\sum_{i=1}^r a_i(w_i^Tz_j)^2 - y_j\right)^2.\]

Note that only the second-layer weight $W$ is trainable. The first layer with weights $R$ acts like a random feature layer that maps $x_i$'s into a new representation $z_i$'s. %, and we determine $R$ in the random initialization. We only train $W$, the parameter for the second layer, when we fix $R$ and $a_i$.

\subsection{Second order stationary points and perturbed gradient descent}
%\haoyu{adding this section to introduce PGD}
Gradient descent converges to a global optimum of a convex function. However, for non-convex objectives, gradient descent is only guaranteed to converge into a first-order stationary point - a point with 0 gradient, which can be a local/global optimum or a saddle point. Our result requires any algorithm that can find a second-order stationary point - a point with 0 gradient and positive definite Hessian. Many algorithms were known to achieve such guarantee\citep{ge2015escaping, sun2015nonconvex, carmon2018accelerated, agarwal2017finding, jin2017escape, jin2017accelerated}. As we require some additional properties of the algorithm (see Section~\ref{sec:proof-sketch-twolayer}), we will adapt the Perturbed Gradient Descent(PGD, \citep{jin2017escape}). See Section~\ref{sec:algdetail} for a detailed description of the algorithm. Here we give the guarantee of PGD that we need. The PGD algorithm requires the function and its gradient to be Lipschitz:% $\varepsilon$-second-order stationary point(Definition \ref{def:stationary-point}) and can escape saddle points efficiently.  

%In this paper, we will use PGD and our theoretical analysis use the convergence theorem in \cite{jin2017escape}. We will introduce the basic definitions and then introduce their convergence theorem of PGD.

\begin{definition}[Smoothness and Hessian Lipschitz]
    A differentiable function $f(\cdot)$ is $\ell$\text{-smooth} if:
    \[\forall x_1,x_2,\ ||\nabla f(x_1) - \nabla f(x_2)|| \le \ell ||x_1 - x_2||.\]
    A twice-differentiable function $f(\cdot)$ is $\rho$\text{-Hessian Lipschitz} if:
    \[\forall x_1,x_2,\ ||\nabla^2 f(x_1) - \nabla^2 f(x_2)|| \le \rho ||x_1 - x_2||.\]
\end{definition}

Under these assumptions, we will consider an approximation for exact second-order stationary point as follows:

\begin{definition}[$\varepsilon$-second-order stationary point]\label{def:stationary-point}
    For a $\rho$-Hessian Lipschitz function $f(\cdot)$, we say that $x$ is an $\varepsilon$\textbf{-second-order stationary point} if:
    \[||\nabla f(x)||\le \varepsilon, \text{and }\lambda_{\min}(\nabla^2f(x)) \ge -\sqrt{\rho \varepsilon}.\]
\end{definition}

%\haoyu{I feel like we can put the pseudo-code for PGD into the appendix if the space is limited.}
%\rong{Yes that makes sense, the box below does not say anything useful.}

\cite{jin2017escape} showed that PGD converges to an $\varepsilon$-second-order stationary point efficiently:
%Generally speaking, with high probability, PGD will converge to an $\varepsilon$-second-order stationary point in polynomial time. The following is the convergence theorem of the PGD(Theorem 3 in \cite{jin2017escape}).

\begin{restatable}[Convergence of PGD (\cite{jin2017escape})]{theorem}{thmpgdconvergence}\label{thm:pgd-convergence}
    Assume that $f(\cdot)$ is $\ell$-smooth and $\rho$-Hessian Lipschitz. Then there exists an absolute constant $c_{\text{max}}$ such that, for any $\delta > 0, \varepsilon 
    \le \frac{\ell^2}{\rho},\Delta_f \ge f(x_0) - f^*$, and constant $c \le c_{\text{max}}$, $PGD(x_0,\ell,\rho,\varepsilon,c,\delta,\Delta_f)$ will output an $\varepsilon$-second-order stationary point with probability $1-\delta$, and terminate in the following number of iterations:
    \[O\left(\frac{\ell(f(x_0) - f^*)}{\varepsilon^2}\log^4\left(\frac{d\ell\Delta_f}{\varepsilon^2\delta}\right)\right).\]
\end{restatable}

%
%\textbf{Our contribution:} 
%
%\textbf{Roadmap: }We introduce our models(2 kinds) and present our main results in Section \ref{sec:model-results}. Then in Section \ref{sec:proof-sketch-twolayer} and Section \ref{sec:proof-sketch-random-feature}, we give proof sketches of our different models. The detailed proof of our main theorems are defered to Appendix.
%
%\subsection{More related works}

\section{Warm-up: Two-layer Net for Fitting Small Training Set}\label{sec:proof-sketch-twolayer}
%\haoyu{I need to change the following presentations based on the structure of this section.}

In this section, we show how the two-layer neural net in Section~\ref{sec:prelim_architecture} trained with perturbed gradient descent can fit any small training set (Theorem \ref{thm:main1:informal}). Our result is based on a characterization of optimization landscape: for small enough $\varepsilon$, every $\varepsilon$-second-order stationary point achieves near-zero training error. We then combine such a result with PGD to show that simple algorithms can always memorize the training data. Detailed proofs are deferred to Section~\ref{sec:twolayerformal} in the Appendix.

%The result is a combination of the results from 2 perspectives: 1. We prove some properties related to the geometry or the optimization landscape of the 2-layer neural network; and 
%2. Given the geometric properties of the 2-layer neural network, some variant of gradient descent will output a \emph{good} solution.

\subsection{Optimization landscape of two-layer neural network}

%\rong{It would be good to add a main theorem here, as this main theorem is going to look much cleaner than the one with optimization, and will allow people to immediately see what we are doing}
%\haoyu{I do not quite understand how to add a cleaner main theorem without optimization. Do you mean that we only state the optimization landscape of 2-layer NN? These results are given in Lemma \ref{lem:smallesteigenvalue} and \ref{lem:spectralnormandfuncvalue} in this section.}
%\rong{What I meant is to combine Lemma 1 and Lemma 2 into a single theorem. Something like ``Every local min of objective function $f$ satisfies $f(x_i) = y_i$''}
%\haoyu{I rewrite the following section. I change the main lemma into a formal lemma, and remove the previous 2 small lemmas into the appendix. I find that presenting the small lemmas will not give any advantage to our presentation, even for the geometric property of function $g$ in the later subsection.}
%In this part, we will show the geometric property of the 2-layer neural network using quadratic activation functions. The main result is given in the following main lemma.

Recall that the two-layer network we consider has $r$-hidden units with bottom layer weights $w_1,w_2,...,w_r$, and the weight for the top layer is set to $a_i = 1$ for $1\le i\le r/2$, and $a_i = -1$ for $r/2+1\le i \le r$. For a set of input data $\{(x_1,y_1),(x_2,y_2),...,(x_n,y_n)\}$, the objective function is defined as \[f(W) = \frac{1}{4n}\sum_{j=1}^n\left(\sum_{i=1}^r a_i(w_i^Tx_j)^2 - y_j\right)^2.\]

With these definitions, we will show that when a point is an approximate second-order stationary point (in fact, we just need it to have an almost positive semidefinite Hessian) it must also have low loss:

\begin{restatable}[Optimization Landscape]{lemma}{lemoptlandscape}\label{lem:geo-property}
Given training data $\{(x_1,y_1),(x_2,y_2),...,(x_n,y_n)\}$,
    Suppose the matrix $X = [x_1^{\otimes 2},\dots,x_n^{\otimes 2}]\in \R^{d^2 \times n}$ has full column rank and the smallest singular value is at least $\sigma$. Also suppose that the number of hidden neurons satisfies $r\ge 2d+2$. Then if $\lambda_{\min}\nabla^2 f(W) \ge -\varepsilon$, the function value is bounded by $f(W) \le \frac{nd\varepsilon^2}{4\sigma^2}$.
\end{restatable}
%\haoyu{I add the formal lemma. Need polish.}

For simplicity, we will use $\delta_j(W) = \sum_{i=1}^r a_i(w_i^T x_j)^2 - y_j$ to denote the {\em residual} for $j$-th data point: the difference between the output of the neural network and the label $y_j$. We will also combine these residuals into a matrix $M(W) :=  \frac{1}{n}\sum_{j=1}^n\delta_j(W) x_jx_j^T$. Intuitively, we first show that when $M(W)$ is large, the smallest eigenvalue of $\nabla^2 f(W)$ is very negative.

\begin{restatable}{lemma}{lemsmallesteigenvalue}\label{lem:smallesteigenvalue}
    When the number of the hidden neurons $r \ge 2d+2$, we have
    \[\lambda_{\min}\nabla^2 f(W) = -\max_i |\lambda_i(M)|,\]
    where $\lambda_{\min}\nabla^2 f(W)$ denotes the smallest eigenvalue of the matrix $\nabla^2 f(W)$ and $\lambda_i (M)$ denotes the $i$-th eigenvalue of the matrix $M$.
\end{restatable}

Then we complete the proof by showing if the objective function is large, $M(W)$ is large.

\begin{restatable}{lemma}{lemspectralnormandvalue}\label{lem:spectralnormandfuncvalue}
    Suppose the matrix $X = [x_1^{\otimes 2},\dots,x_n^{\otimes 2}]\in \R^{d^2 \times n}$ has full column rank and the smallest singular value is at least $\sigma$. Then if the spectral norm of the matrix $M = \frac{1}{n}\sum_{j=1}^n \delta_j x_jx_j^T$ is upper bounded by $\lambda$, the function value is bounded by
    \[f(W) \le \frac{nd\lambda^2}{4\sigma^2}.\]
\end{restatable}

Combining the two lemmas, we know $f(W)$ is bounded when the point has almost positive Hessian, therefore every $\varepsilon$-second-order stationary point must be near-optimal. 

% we will show that 1. when any of the $\delta_j$'s are large, then $M(W)$ is large;
%2. when $M(W)$ is large, the smallest eigenvalue of $\nabla^2 f(W)$ is very negative. Therefore when $\lambda_{min}(\nabla^2 f(W))$ is close to 0, the residuals are all small and the loss is also small.

%To prove the main Lemma \ref{lem:geo-property}), we will show that: 1. when the width of the neural network is \emph{large enough}, then the smallest eigenvalue of $\nabla^2 f(W)$ equals to the negative of the spectral norm of matrix $M(W)$; and 2. when the samples are in \emph{general position}, then then every point $W$ such that $M(W)$ has small spectral norm will have small function value $f(W)$.

\subsection{Optimizing the two-layer neural net}
In this section, we show how to use PGD to train our two-layer neural network.

%\rong{These definitions should be moved to preliminaries when you introduce PGD}
%\haoyu{Move the definition and main PGD theorem into the preliminary}

%\rong{First you need to say something about difficulty - for a reader it would seem that the main theorem in the previous subsection, together with the guarantee of PGD, would immediately imply the main theorem, why is that not true? Then you can propose how we fix things by adding the regularizer.}
%\haoyu{I state the difficulty more explicitly below, please see if that illustration is OK.}

%Combined with the result in the previous section, we know that if the matrix $X = [x_1^{\otimes 2},\dots,x_n^{\otimes 2}]\in \R^{d^2 \times n}$ has full column rank and the smallest singular value is at least $\sigma$, we will know that every point such that $|\lambda_{\min}\nabla^2 f(W)|$ is small will have small function value $f(W)$.
Given the property of the optimization landscape for $f(W)$, it is natural to directly apply PGD to find a second-order stationary point. However, this is not enough since the function $f$ does not have bounded smoothness constant and Hessian Lipschitz constant (its Lipschitz parameters depend on the norm of $W$), and without further constraints, PGD is not guaranteed to converge in polynomial time. In order to control the Lipschitz parameters, we note that these parameters are bounded when the norm of $W$ is bounded (see Lemma~\ref{lem:smoothness-lipschitz-hessian} in appendix). Therefore we add a small regularizer term to control the norm of $W$. More concretely, we optimize the following objective
\[g(W) = f(W) + \frac{\gamma}{2}||W||_F^2.\]
We want to use this regularizer term to show that: 1. the optimization landscape is preserved: for appropriate $\gamma$, any $\varepsilon$-second-order stationary point of $g(W)$ will still give a small $f(W)$; and 2. During the training process of the 2-layer neural network, the norm of $W$ is bounded, therefore the smoothness and Hessian Lipschitz parameters are bounded. Then, the proof of Theorem \ref{thm:main1:informal} just follows from the combination of Theorem \ref{thm:pgd-convergence} of PGD and the result of the geometric property.

The first step is simple as the regularizer only introduces a term $\gamma I$ to the Hessian, which increases all the eigenvalues by $\gamma$. Therefore any $\varepsilon$-second-order stationary point of $g(W)$ will also lead to the fact that $|\lambda_{\min}\nabla^2 f(W)|$ is small, and hence $f(W)$ is small by Lemma~\ref{lem:geo-property}. %Then by the results from the previous subsection, we have $f(W)$ is small.

%Next we bound the smoothness and Hessian Lipschitzness parameters. To do that, observe that both quantities are bounded with respect to the norm of $W$:
%Next we will give a proof sketch of why the training process will not escape some area. We first have the following technical lemma,

%Given this lemma, it suffices to show that the training process using PGD will not escape from the area $\{W:||W||_F^2 \le \Gamma\}$ with some $\Gamma$. We will actually show that the function value $g(W)$ always smaller than $\gamma \Gamma/2$, which implies $\|W\|_F^2 \le \frac{2}{\gamma} g(W) \le \Gamma$. % Since we have the regularizer term in our cost function $g(W)$, it suffices to show that the function value will not exceed some value. 

For the second step, note that in order to show the training process using PGD will not escape from the area $\{W:||W||_F^2 \le \Gamma\}$ with some $\Gamma$, it suffices to bound the function value $g(W)$ by $\gamma \Gamma/2$, which implies $\|W\|_F^2 \le \frac{2}{\gamma} g(W) \le \Gamma$. To bound the function value we use properties of PGD: for a gradient descent step, since the function is smooth in this region, the function value always decreases; for a perturbation step, the function value can increase, but cannot increase by too much. %We 
%Since by the basic property of smooth functions, we know that in most of the time, the function value will decrease when applying PGD, and we can also show that the perturbation will not increase the function value too much. 
Using mathematical induction, we can show that the function value of $g$ is smaller than some fixed value(related to the random initialization but not related to time $t$) and will not escape the set $\{W:||W||_F^2 \le \Gamma\}$ for appropriate $\Gamma$.

Using PGD on function $g(W)$, we have the following main theorem for the 2-layer neural network.

\begin{restatable}[Main theorem for 2-layer NN]{theorem}{thmmainthmtwolayer}\label{thm:main-theorem-twolayer}
     Suppose the matrix $X = [x_1^{\otimes 2},\dots,x_n^{\otimes 2}]\in \R^{d^2 \times n}$ has full column rank and the smallest singular value is at least $\sigma$. Also assume that we have $||x_j||_2 \le B$ and $|y_j| \le Y$ for all $j \le n$. We choose our width of neural network $r\ge 2d+2$ and we choose $\rho = (6B^4\sqrt{2(f(0) + 1)})\left(nd/(\sigma^2\varepsilon)\right)^{1/4}$, $\gamma = \left(\sigma^2\varepsilon/nd\right)^{1/2}$, and  $\ell = \max\{(3B^4\frac{2(f(0) + 1)}{\gamma}+YB^2+\gamma),1\}$. Then there exists an absolute constant $c_{\text{max}}$ such that, for any $\delta > 0,\Delta \ge f(0) + 1$, and constant $c \le c_{\text{max}}$, $PGD(0,\ell,\rho,\varepsilon,c,\delta,\Delta)$ on $W$ will output an parameter $W^*$ such that with probability $1-\delta$, $f(W^*) \le \varepsilon$ when the algorithm terminates in the following number of iterations:
    \[O\left(\frac{B^8\ell(nd)^{5/2}(f(0)+1)^2}{\sigma^{5}\varepsilon^{5/2}}\log^4\left(\frac{Bnrd\ell\Delta(f(0)+1)}{\varepsilon^2\delta\sigma}\right)\right).\]
\end{restatable}
%\haoyu{The parameters in the previous theorem is fixed.}

\section{Three-Layer Net for Fitting Larger Training Set}\label{sec:proof-sketch-random-feature}
%\haoyu{I feel like the main theorems in this section do not help us to present our result. The main theorems use a lot of space and the parameters are intimidating, and we already state the formal one in the previous section. I suggest that we give 2 not-so-formal main theorems, only saying that with proper parameters and treat $p$ as a constant, etc. Then we state our most formal theorem in the appendix.}

In this section, we show how a three-layer neural net can fit a larger training set (Theorem \ref{thm:main2:informal}). The main limitation of the two-layer architecture in the previous section is that the activation functions are quadratic. No matter how many neurons the hidden layer has, the whole network only captures a quadratic function over the input, and cannot fit an arbitrary training set of size much larger than $d^2$. On the other hand, if one replaces the quadratic activation with other functions, it is known that even two-layer neural networks can have bad local minima\citep{safran2018spurious}.

To address this problem, the three-layer neural net in this section uses the first-layer as a random mapping of the input. The first layer is going to map inputs $x_i$'s into $z_i$'s of dimension $k$ (where $k = \Theta(\sqrt{n})$). If $z_i$'s satisfy the requirements of Theorem~\ref{thm:main-theorem-twolayer}, then we can use the same arguments as the previous section to show perturbed gradient descent can fit the training data.

%Given the result of a 2-layer neural net, we only have to show that the outputs of the random feature layer satisfy: 1. The norm of each output is upper %2. The outputs are in \emph{general position}.

%We will first show our result in a smoothed analysis framework: we will add small perturbation to the samples. Then, we will show another result without the perturbation, but the result is weaker.

%\rong{For this section, I would give the smoothed analysis result first as the result is easier to state. Then you can say if you prefer a deterministic condition, we have another result with a slightly weaker guarantee.}

%\rong{The ordering of lemmas here is not optimal, you should focus on one setting first (say smoothed analysis), introduce the two lemmas for smoothed analysis, and point out why these two lemmas combined with the main theorem before would immediately give what we need. Then you start a new subsection for the other two lemmas.}

%\haoyu{Currently I change the order and divide this section into 2 subsections. In each section, I introduce the main theorem first and show how to derive the main theorem.}

%\haoyu{I guess we need to add some general description before the following 2 subsections.}

%\subsection{Results with adding perturbation on data}\label{sec:proof-sketch-3-layer-perturb}
%In this subsection, we first focus on the case when we add perturbation on the samples.  
We prove our main result in the smoothed analysis setting, which is a popular approach for going beyond worst-case. Given any worst-case input $\{x_1,x_2,...,x_n\}$, in the smoothed analysis framework, these inputs will first be slightly perturbed before given to the algorithm. More specifically, let $\bar x_j = x_j + \tilde x_j$, where $\tilde x_j\in \mathbb R^{d}$ is a random Gaussian vector following the distribution of $\cN(\mathbf 0, v \mathbf I)$. Here the amount of perturbation is controlled by the variance $v$. The final running time for our algorithm will depend inverse polynomially on $v$. Note that on the other hand, the network architecture and the number of neurons/parameters in each layer does not depend on $v$.

Let $\{z_1,z_2,...,z_n\}$ denote the output of the first layer with $(z_j)_i = (r_i^T \bar x_j)^p (j = 1,2,...,n)$, we first show that $\{z_j\}$'s satisfy the requirement of Theorem~\ref{thm:main-theorem-twolayer}: %and the rest are the same as the case without perturbation on $x_j$. Then we have the following theorem for 3-layer neural network with perturbation on $\{x_j\}_{j\le n}$.

\begin{restatable}{lemma}{lemsmallestsingluarperturb}\label{lem:smallest-singular-perturb}
    Suppose $k \le O_p(d^p)$ and ${k+1 \choose 2} > n$, let $\bar{x}_j = x_j +\tilde{x}_j$ be the perturbed input in the smoothed analysis setting, where $\tilde{x}_j \sim \cN(\mathbf 0, v \mathbf I)$, let $\{z_1,z_2,...,z_n\}$ be the output of the first layer on the perturbed input ($(z_j)_i = (r_i^T \bar x_j)^p$). Let $Z\in \R^{k^2\times n}$ be the matrix whose $j$-th column is equal to $z_j^{\otimes 2}$, then with probability at least $1-\delta$, the smallest singular value of $Z$ is at least $\Omega_p(v^p\delta^{4p}/n^{2p+1/2}k^{4p})$.
\end{restatable}

%\rong{I did not change the restatable lemma below as it might incur future changes. However, this is how you should state the lemma. When you read the previous version you need to understand what is this $Z$ matrix and why do we care about its smallest singular value.}

%\rong{Previously the statement is $\sqrt{n} \le k$ but that is almost surely incorrect, you need $n < {k \choose 2}$ at the very least, and we probably lose a bit more in the decoupling?}
%\haoyu{We will fix them asap, but I guess that does not make a lot of change}

This lemma shows that the output of the first layer ($z_j$'s) satisfies the requirements of Theorem~\ref{thm:main-theorem-twolayer}. With this lemma, we can prove the main theorem of this section:

\begin{restatable}[Main theorem for 3-layer NN]{theorem}{thmrandomfeaturewithperturb}\label{thm:main-theorem-random-feature-with-perturbation}
%    Suppose that we have $||x_j||_2 \le B$ and $|y_j| \le Y$ for all $j \le n$. For the 3-layer NN structured mentioned above with $D_d^p\geq k\geq \sqrt{n}$ and $r=2k+2$, we perturb $x_j$ as $\bar{x_j}=x_j+\tilde{x_j}$ with i.i.d. $\tilde{x_j}\sim\cN(\mathbf 0, v\text{I})$, randomly generate $r_i\sim \cN(\mathbf 0, \text{I})$, and choose $B_0=\sqrt{k}\left(2(B + 2\sqrt{vd\ln ((k+n)d\delta^{-1/2})})\sqrt{d\ln ((k+n)d\delta^{-1/2})}\right)^p$, $\sigma=\left(\frac{(D_d^p-k)^2(k^2-n)}{[(2p)!]^2[(4p)!]}\right)^{1/4}\frac{v^{p}\delta^{4p}}{p^{4p}n^{2p+1/2}k^{2p+1}}$, $\rho = (6B_0^4\sqrt{2(f(0) + 1)})\left(nk/(\sigma^2\varepsilon)\right)^{1/4}$, $\gamma = \left(\sigma^2\varepsilon/nk\right)^{1/2}$, and  $\ell = \max\{(3B_0^4\frac{2(f(0) + 1)}{\gamma}+YB_0^2+\gamma),1\}$. Then there exists an absolute constant $c_{\text{max}}$ such that, for any $\delta > 0,\Delta \ge f(0) + 1$, and constant $c \le c_{\text{max}}$, $PGD(0,\ell,\rho,\varepsilon,c,\delta,\Delta)$ optimizing $W$ in the network will output an parameter $W^*$ such that with probability $1-\delta$, $f(W^*) \le \varepsilon$ when the algorithm terminates in the following number of iterations:
%    \[O\left(\frac{B_0^8\ell(nk)^{5/2}(f(0)+1)^2}{\sigma^{5}\varepsilon^{5/2}}\log^4\left(\frac{B_0nk^2\ell\Delta(f(0)+1)}{\varepsilon^2\delta\sigma}\right)\right)=.\]   
Suppose the original inputs satisfy $\|x_j\|_2\le 1, |y_j|\le 1$, inputs $\bar x_j = x_j+\tilde{x}_j$ are perturbed by $\tilde{x}_j\sim \cN(0,vI)$, with probability $1-\delta$ over the random initialization, for $k = 2\lceil \sqrt{n} \rceil$, perturbed gradient descent on the second layer weights achieves a loss $f(W^*) \le \epsilon$ in $O_p(1)\cdot \frac{(n/v)^{O(p)}}{\epsilon^{5/2}} \log^4(n/\epsilon)$ iterations.
\end{restatable}

Using different tools, we can also prove a similar result without the smoothed analysis setting:

\begin{restatable}{theorem}{thmdeterministic}\label{thm:deterministic} Suppose the matrix $X = [x_1^{2p},...,x_n^{2p}]\in \R^{d^{2p}\times n}$ has full column rank, and smallest singular value at least $\sigma$. Choose $k = O_p(d^p)$, with high probability perturbed gradient descent on the second layer weights achieves a loss $f(W^*) \le \epsilon$ in $O_p(1)\cdot \frac{(n)^{O(p)}}{\sigma^5\epsilon^{5/2}} \log^4(n/\epsilon)$ iterations.
\end{restatable}

When the number of samples $n$ is smaller than $d^{2p}/(2p)!$, one can choose $k = O_p(d^p)$, in this regime the result of Theorem~\ref{thm:deterministic} is close to Theorem~\ref{thm:main-theorem-random-feature-with-perturbation}. However, if $n$ is just larger, say $n = d^{2p}$, one may need to choose $k = O_p(d^{p+1})$, which gives sub-optimal number of neurons and parameters.

%\rong{It's OK to stop the section here. We were including too much details.}
%\haoyu{I first move the original writings into a single section in the appendix. I will move them into the proofs later.}

\section{Experiments}
In this section, we validate our theory using experiments. Detailed parameters of the experiments as well as more result are deferred to Section~\ref{sec:more-experiments} in Appendix. %some experiment results to verify our network structure. Due to the space limitation, we will briefly introduce how we do our experiments and shows the results. Further details including the detailed explanation of our experiments and the selection of hyper-parameters are deferred into Appendix.

\paragraph{Small Synthetic Example} We first run gradient descent on a small synthetic data-set, which fits into the setting of Theorem~\ref{thm:main-theorem-twolayer}. % This experiment aims to show by following the settings in Theorem \ref{thm:main-theorem-twolayer}, we can memorize the data with 2-layer network. 
Our training set, including the samples and the labels, are generated from a fixed normalized uniform distribution(random sample from a hypercube and then normalized to have norm 1). As shown in Figure \ref{fig:random-sample}, simple gradient descent can already memorize the training set. %our network structure can efficiently memorize all the training set.

\begin{figure}[ht!]
    \centering
    \includegraphics[width=2.5in]{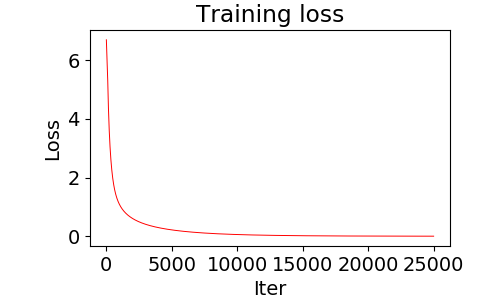}
    \caption{Training loss for random sample experiment}
    \label{fig:random-sample}
\end{figure}

\paragraph{MNIST Experiment} We also show how our architectures (both two-layer and three-layer) can be used to memorize MNIST. For MNIST, we use a squared loss between the network's prediction and the true label (which is an integer in $\{0,1,...,9\}$). For the two-layer experiment, we use the original MNIST dataset, with a small Gaussian perturbation added to the data to make sure the condition in Theorem~\ref{thm:main-theorem-twolayer} is satisfied. For the three-layer experiment, we use PCA to project MNIST images to 100 dimensions (so the two-layer architecture will no longer be able to memorize the training set). See Figure~\ref{fig:original-label} for the results. In this part, we use ADAM as the optimizer to improve convergence speed, but as we discussed earlier, our main result is on the optimization landscape and the algorithm is flexible.

%In this part, we do numerical experiments to show that our network structure can memorize the original MNIST training dataset. We do 2 experiments, the one to verify our 2-layer network structure, and the other to verify our 3-layer structure(Figure \ref{fig:original-label}). Our simulation is a little \emph{heuristic}, since we do not strictly follow all the settings in Theorem \ref{thm:main-theorem-twolayer} and \ref{thm:main-theorem-random-feature-with-perturbation}. 

\begin{figure}[!th]
    \centering
    \subfloat[Two-layer network with perturbation on input]{\includegraphics[width=2.5in]{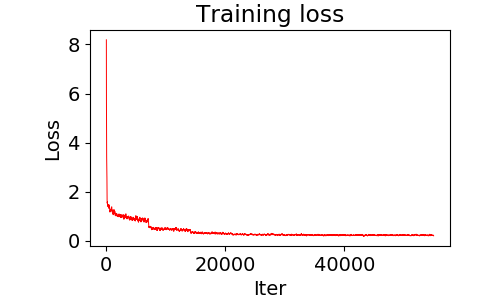}}
    \subfloat[Three-layer network on top 100 PCA directions]{\includegraphics[width=2.5in]{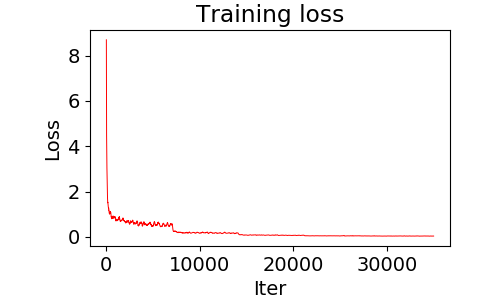}}
    \caption{MNIST with original label}
    \label{fig:original-label}
\end{figure}

\paragraph{MNIST with random label} We further test our results on MNIST with random labels to verify that our result does not use any potential structure in the MNIST datasets. The setting is exactly the same as before. As shown in Figure \ref{fig:random-label}, the training loss can also converge.

\begin{figure}[!th]
    \centering
    \subfloat[Two-layer network with perturbation on input]{\includegraphics[width=2.5in]{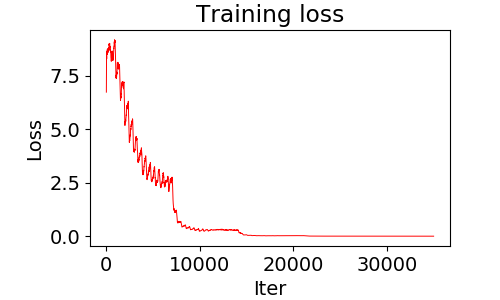}}
    \subfloat[Three-layer network on top 100 PCA directions]{\includegraphics[width=2.5in]{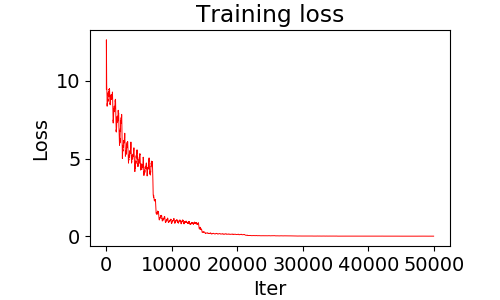}}
    \caption{MNIST with random label}
    \label{fig:random-label}
\end{figure}

\section{Conclusion} 
In this paper, we showed that even a mildly overparametrized neural network can be trained to memorize the training set efficiently. The number of neurons and parameters in our results are tight (up to constant factors) and matches the bounds in \cite{yun2018small}. There are several immediate open problems, including generalizing our result to more standard activation functions and providing generalization guarantees. More importantly, we believe that the mildly overparametrized regime is more realistic and interesting compared to the highly overparametrized regime. We hope this work would serve as a first step towards understanding the mildly overparametrized regime for deep learning.

\newpage

\bibliography{ref}

\newpage

\appendix

\section{More Experiments and Detailed Experiment Setup}
\label{sec:more-experiments}

\subsection{Experiments setup}
In this section, we introduce the experiment setup in detail.

\paragraph{Small Synthetic Example} We generate the dataset in the following way: We first set up a random matrices $X\in\R^{N\times d}$(samples), where $N$ is the number of samples, $d$ is the input dimension and $Y\in\R^{N}$(labels). Each entry in $X$ or $Y$ follows a uniform distribution with support $[-1,1]$. Each entry is independent from others. Then we normalize the dataset $X$ such that each row in $X$ has norm $1$, denote the normalized dataset as $\hat X = [\hat x_1,\dots,\hat x_N]^T$. Then we compute the smallest singular value for the matrix $[\hat x_1^{\otimes 2},\dots,\hat x_N^{\otimes 2}]^T$, and we feed the normalized dataset $\hat X$ into the two-layer network(Section \ref{sec:prelim_architecture}) with $r$ hidden neurons. We select all the parameters as shown in Theorem \ref{thm:main-theorem-twolayer}, and plot the function value for $f(\cdot)$.

In our experiment for the small artificial random dataset, we choose $N = 300,d = 100$, and $r = 300$.

\paragraph{MNIST experiments}

For MNIST, we use a squared loss between the network's prediction and the true label (which is an integer in $\{0,1,...,9\}$). 

For the first two-layer network structure, we first normalize the samples in MNIST dataset to have norm 1. Then we set up a two-layer network with quadratic activation with $r = 3000$ hidden neurons (note that although our theory suggests to choose $r = 2d+2$, having a larger $r$ increases the number of decreasing directions and helps optimization algorithms in practice). For these experiments, we use Adam optimizer\citep{kingma2014adam} with batch size 128, initial learning rate 0.003, and decay the learning rate by a factor of 0.3 every 15 epochs (we find that the learning rate-decay is crucial for getting high accuracy). 

We run the two-layer network in two settings, one for the original MNIST data, and one for the MNIST data with a small Gaussian noise (0.01 standard deviation per coordinate). The perturbation is added in order for the conditions in Theorem~\ref{thm:main-theorem-twolayer} to hold.

%Instead of following to the statement in Theorem \ref{thm:main-theorem-twolayer}, we use Adam optimizer with learning rate decay mechanism. We also remove the very small regularization term and heuristically select the hyper-parameters including the learning rate. Besides, we also test the case when we add a small perturbation onto the samples.

For the three-layer network structure, we first normalize the samples in MNIST dataset with norm 1. Then we do the PCA to project it into a 100-dimension subspace. We use $D = [x_1,\dots,x_n]$ to denote this dataset after PCA. Note that the original 2-layer network may not apply to this setting, since now the matrix $X = [x_1^{\otimes 2},\dots, x_n^{\otimes 2}]$ does not have full column rank($60000 > 100^2)$. We then add a small Gaussian perturbation to $\tilde D\sim \cN(0,\sigma_1^2)$ to the sample matrix $D$ and denote the perturbed matrix $\bar D = [\bar x_1,\dots,\bar x_n]$. We then randomly select a matrix $Q\sim \cN(0,\sigma_2^2)^{k\times d}$ and compute the random feature $z_j = (Q\bar x_j)^2$, where $(\cdot)^2$ denote the element-wise square. Then we feed this sample into the 2-layer neural network with hidden neuron $d$. Note that this is equivalent to our three-layer network structure in Section \ref{sec:prelim_architecture}. In our experiments, $k = 750, r = 3000, \sigma_1 = 0.05, \sigma_2 = 0.15$.

\paragraph{MNIST with random labels} These experiments have exactly the same set-up as the original MNIST experiments, except that the labels are replaced by a random number in \{0,1,2,...,9\}.

%For the parameters in the 3-layer experiment, the perturbation use Gaussian random variable with standard deviation $0.01$($\sigma_1 = 0.01$) and random weights use Gaussian random variable with standard deviation $0.2$($\sigma_2 = 0.2$). The output size of the random feature layer is chosen to be $k = 500$, which is smaller than the original input size 784. Then the number of hidden neurons is set to be $r = 3000$. 
%We initialize our network(only the trained part) using Xavier Normal initialization. As for the optimizer, we use Adam optimizer with initial learning rate 0.003 and the batch size is 128. We decrease the learning rate by a factor of 0.3 after 15 epochs.

%We also test our network structure on the original MNIST dataset with further perturbations.

%\paragraph{MNIST with random label}

\subsection{Experiment Results}
In this section, we give detailed experiment results with bigger plots. For all the training loss graphs, we record the training loss for every 5 iterations. Then for the $i$th recorded loss, we average the recorded loss from $i-19$th to $i$th and set it as the average loss at $(5i)$th iteration. Then we take the logarithm on the loss and generated the training loss graphs.

\paragraph{Small Synthetic Example}

%\haoyu{I just list the large result graph in the following place}

\begin{figure}[h!]
    \centering
    \includegraphics[width=5in]{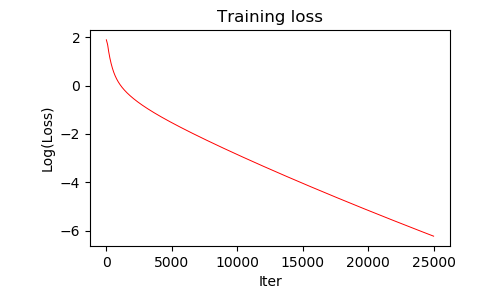}
    \caption{Synthetic Example}
    \label{fig:random-sample-large}
\end{figure}

As we can see in Figure~\ref{fig:random-sample-large} the loss converges to 0 quickly.

\paragraph{MNIST experiments with original labels}

\begin{figure}[ht!]
    \centering
    \includegraphics[width=5in]{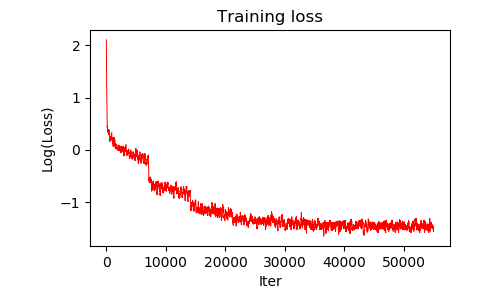}
    \caption{Two-layer network on original MNIST}
    \label{fig:2-layer-original-large}
\end{figure}

\begin{figure}[ht!]
    \centering
    \includegraphics[width=5in]{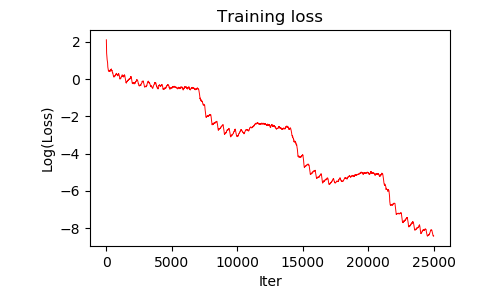}
    \caption{Two-layer network on MNIST, with noise std 0.01}
    \label{fig:2-layer-perturbed-large}
\end{figure}

First we compare Figure~\ref{fig:2-layer-original-large} and Figure~\ref{fig:2-layer-perturbed-large}. In Figure~\ref{fig:2-layer-original-large}, we optimize the two-layer architecture with original input/labels. Here the loss decreases to a small value ($\sim 0.1$), but the decrease becomes slower afterwards. This is likely because for the matrix $X$ defined in Theorem~\ref{thm:main-theorem-twolayer}, some of the directions have very small singular values, which makes it much harder to correctly optimize for those directions. In Figure~\ref{fig:2-layer-perturbed-large}, after adding the perturbation the smallest singular value of the matrix $X$ becomes better, and as we can see the loss decreases geometrically to a very small value ($<1e-5$).

A surprising phenomenon is that even though we offer no generalization guarantees, the network trained as in Figure~\ref{fig:2-layer-original-large} has an MSE error of 1.21 when tested on test set, which is much better than a random guess (recall the range of labels is 0 to 9). This is likely due to some implicit regularization effect \citep{gunasekar2017implicit, li2018algorithmic}.

For three-layer networks, in Figure~\ref{fig:3-layer-moise-large} we can see even though we are using only the top 100 PCA directions, the three-layer architecture can still drive the training error to a very low level.

%\begin{figure}[ht!]
%    \centering
%    \includegraphics[width=5in]{pca35000_500.png}
%    \caption{3-layer network}
%    \label{fig:3-layer-original-large}
% Rong: This also has noise but just a bit smaller.
%\end{figure}

\begin{figure}[ht!]
    \centering
    \includegraphics[width=5in]{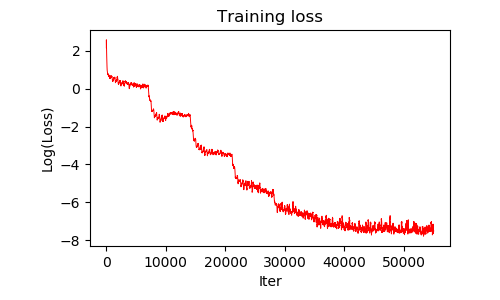}
    \caption{Three-layer network with top 100 PCA directions of MNIST, 0.05 noise per direction}
    \label{fig:3-layer-moise-large}
\end{figure}

\paragraph{MNIST with random label}

When we try to fit random labels, the original MNIST input does not work well. We believe this is again because there are many small singular values for the matrix $X$ in Theorem~\ref{thm:main-theorem-twolayer}, so the data does not have enough effective dimensions fit random labels. The reason that it was still able to fit the original labels to some extent (as in Figure~\ref{fig:2-layer-original-large}) is likely because the original label is correlated with some features of the input, so the original label is less likely to fall into the subspace with smaller singular values. Similar phenomenon was found in \cite{arora2019fine}.

Once we add perturbation, for two-layer networks we can fit the random label to very high accuracy, as in Figure~\ref{fig:2-layer-perturbed-random-sample-large}. The performance for three-layer network in Figure~\ref{fig:3-layer-random-label-large} is also similar to Figure~\ref{fig:3-layer-moise-large}.

\begin{figure}[ht!]
    \centering
    \includegraphics[width=5in]{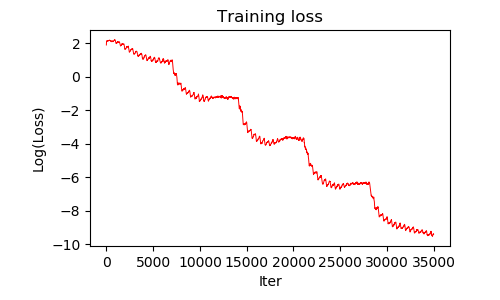}
    \caption{Two-layer network on MNIST, with noise std 0.01, random labels}
    \label{fig:2-layer-perturbed-random-sample-large}
\end{figure}

\begin{figure}[ht!]
    \centering
    \includegraphics[width=5in]{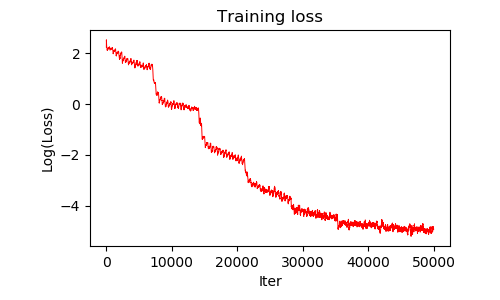}
    \caption{Three-layer network with top 100 PCA directions of MNIST, 0.05 noise per direction, random labels}
    \label{fig:3-layer-random-label-large}
\end{figure}

%\paragraph{Random sample and random label experiment} Because in this experiment, we strictly follow the settings in Theorem \ref{thm:main-theorem-twolayer}, the plot for the training loss(Figure \ref{fig:random-sample-large}) shows the logarithm of function value of $f(\cdot)$ instead of $g(\cdot)$, which is the function to train. As shown in Figure \ref{fig:random-sample-large}, the loss function converges really fast and the network indeed memorize the data.

%\paragraph{Original MNIST}

%\paragraph{MNIST with random label}

\section{Detailed Description of Perturbed Gradient Descent}
\label{sec:algdetail}
In this section we give the pseudo-code of the Perturbed Gradient Descent algorithm as in \cite{jin2017escape}, see Algorithm~\ref{alg:pgd}. The algorithm is quite simple: it just runs the standard gradient descent, except if the loss has not decreased for a long enough time, it adds a perturbation. The perturbation allows the algorithm to escape saddle points. Note that we only use PGD algorithm to find a second-order stationary point. Many other algorithms, including stochastic gradient descent and accelerated gradient descent, are also known to find a second-order stationary point efficiently. All these algorithms can be used for our analysis.

\begin{algorithm}[!ht]
    \caption{Perturbed Gradient Descent}
    \label{alg:pgd}
    \begin{algorithmic}[1]
        \Require $x_0,\ell,\rho,\varepsilon,c,\delta,\Delta_f$.
        \State $\chi\leftarrow 3\max\left\{\log\left(\frac{d\ell \Delta_f}{c\varepsilon^2 \delta}\right),4\right\},\eta\leftarrow\frac{c}{\ell},r\leftarrow\frac{\sqrt{c}\varepsilon}{\chi^2\ell},g_{\text{thres}}\leftarrow \frac{\sqrt{c}\varepsilon}{\chi^2},f_{\text{thres}}\leftarrow\frac{c\sqrt{\varepsilon^3}}{\chi^3\sqrt{\rho}},t_{\text{thres}}\leftarrow \frac{\chi\ell}{c^2\sqrt{\rho\varepsilon}}$
        \State $t_{\text{noise}} \leftarrow -t_{\text{thres}} - 1$
        \For{$t=0,1,\dots$}
            \If{$||\nabla f(x_t)||\le g_{\text{thres}}$ and $t - t_{\text{noise}} > t_{\text{thres}}$}
                \State $\tilde x_t \leftarrow x_t, t_{\text{noise}}\leftarrow t$
                \State $x_t \leftarrow \tilde x_t + \xi_t$, where $\xi_t$ is drawn uniformly from $\mathbb B_0(r)$.
            \EndIf
            \If{$t - t_{\text{noise}} = t_{\text{thres}}$ and $f(x_t) - f(\tilde x_{t_{\text{noise}}}) > -f_{\text{thres}}$}
                \State\Return $\tilde x_{t_{\text{noise}}}$
            \EndIf
            \State $x_{t+1}\leftarrow x_t - \eta\nabla f(x_t)$
        \EndFor
    \end{algorithmic}
\end{algorithm}

%\rong{What is the pseudo-code? It doesn't do anything. Also nothing here is defined. We need a better presentation of the PGD algorithm. I remember in my paper with Chi we had another algorithm box that uses all of the quantities that were defined here. We should also talk about how these parameters are set.}\haoyu{Sorry for the pseudo-code. That box serves as a placeholder and I have just fill the blanks. Do we need to add some illustrations towards PGD, or just leave the box here?}

\section{Gradient and Hessian of the Cost Function}
%\haoyu{Pass this section over. I just fix some over-full formula boxes.}

Before we prove any of our main theorems, we first compute the gradient and Hessian of the functions $f(W)$ and $g(W)$. In our training process, we need to compute the gradient of function $g(W)$, and in the analysis for the smoothness and Hessian Lipschitz constants, we need both the gradient and Hessian.

Recall that given the samples and their corresponding labels $\{(x_j,y_j)\}_{j\le n}$, we define the cost function of the neural network with parameters $W = [w_1,\dots,w_r]\in\R^{d\times r}$,
\[f(W) = \frac{1}{4n}\sum_{j=1}^n\left(\sum_{i=1}^r a_i(w_i^Tx_j)^2 - y_j\right)^2.\]
Given the above form of the cost function, we can write out the gradient and the hessian with respect to $W$. We have the following gradient,
\begin{align*}
    \frac{\partial f(W)}{\partial w_k} =& \frac{1}{4n}\sum_{j=1}^n 2\left(\sum_{i=1}^r a_i(w_i^Tx_j)^2 - y_j\right) \cdot 2 a_k (w_k^Tx_j) x_j\\
    =& \frac{a_k}{n}\sum_{j=1}^n \left(\sum_{i=1}^r a_i(w_i^Tx_j)^2 - y_j\right)x_jx_j^Tw_k.
\end{align*}
and $\frac{\partial^2 f(W)}{\partial w_{k_1}\partial w_{k_2}} =$
\[ \left\{\begin{aligned}
   \frac{a_{k_1}}{n}\sum_{j=1}^n \left(\sum_{i=1}^r a_i(w_i^Tx_j)^2 - y_j\right)x_jx_j^T + \frac{2a_{k_1}a_{k_2}}{n}\sum_{j=1}^n(x_j^Tw_{k_1})(x_j^Tw_{k_2})x_jx_j^T&,\ \text{if}\  k_1 = k_2 \\
   \frac{2a_{k_1}a_{k_2}}{n}\sum_{j=1}^n(x_j^Tw_{k_1})(x_j^Tw_{k_2})x_jx_j^T &,\ \text{if}\  k_1 \neq k_2
\end{aligned}\right.\]
In the above computation, $\frac{\partial f(W)}{\partial w_k}$ is a column vector and $\frac{\partial^2 f(W)}{\partial w_{k_1}\partial w_{k_2}}$ is a square matrix whose different rows means the derivative to elements in $w_{k_2}$ and different columns represent the derivative to elements in $w_{k_1}$. Then, given the above formula, we can write out the quadratic form of the hessian with respect to the parameters $Z = [z_1,z_2,\dots,z_r]\in \R^{d\times r}$,
\begin{align*}
    &\nabla^2 f(W)(Z,Z)\\
    =& \sum_{k=1}^r z_k^T\left(\frac{a_{k}}{n}\sum_{j=1}^n \left(\sum_{i=1}^r a_i(w_i^Tx_j)^2 - y_j\right)x_jx_j^T\right)z_k \\
    &\quad + \sum_{1\le k_1,k_2\le r}w_{k_2}^T\left(\frac{2a_{k_1}a_{k_2}}{n}\sum_{j=1}^n(x_j^Tw_{k_1})(x_j^Tw_{k_2})x_jx_j^T\right)w_{k_1} \\
    =& \sum_{k=1}^r z_k^T\left(\frac{a_{k}}{n}\sum_{j=1}^n \left(\sum_{i=1}^r a_i(w_i^Tx_j)^2 - y_j\right)x_jx_j^T\right)z_k + \frac{2}{n}\sum_{j=1}^n\left(\sum_{i=1}^r a_i w_i^Tx_jx_j^Tz_i\right)^2.
\end{align*}

In order to train this neural network in polynomial time, we need to add a small regularizer to the original ocst function $f(W)$. Let
\[g(W) = f(W) + \frac{\gamma}{2}||W||_F^2,\]
where $\gamma$ is a constant. Then we can directly get the gradient and the hessian of $g(W)$ from those of $f(W)$. We have
\begin{align*}
    \nabla_{w_k} g(W) =& \frac{a_k}{n}\sum_{j=1}^n \left(\sum_{i=1}^r a_i(w_i^Tx_j)^2 - y_j\right)x_jx_j^Tw_k + \gamma w_k \\
    \nabla^2_W g(W)(Z,Z) =& \sum_{k=1}^r z_k^T\left(\frac{a_{k}}{n}\sum_{j=1}^n \left(\sum_{i=1}^r a_i(w_i^Tx_j)^2 - y_j\right)x_jx_j^T\right)z_k\\
    &\quad+ \frac{2}{n}\sum_{j=1}^n\left(\sum_{i=1}^r a_i w_i^Tx_jx_j^Tz_i\right)^2 + \gamma ||Z||_F^2.
\end{align*}
For simplicity, we can use $x_j^TWAW^Tx_j-y_j$ to denote $(\sum_{i=1}^r a_i(w_i^Tx_j)^2 - y_j$, where $A$ is a diagonal matrix with $A_{ii} = a_i$. Then we have
\begin{align*}
    \nabla_W g(W) =& \frac{1}{n}\sum_{j=1}^n \left(x_j^TWAW^Tx_j-y_j\right)x_jx_j^TWA + \gamma W \\
    \nabla^2_W g(W)(Z,Z) =& \sum_{k=1}^r z_k^T\left(\frac{a_{k}}{n}\sum_{j=1}^n \left(x_j^TWAW^Tx_j-y_j\right)x_jx_j^T\right)z_k\\
    &\quad+ \frac{2}{n}\sum_{j=1}^n\left(\sum_{i=1}^r a_i w_i^Tx_jx_j^Tz_i\right)^2 + \gamma ||Z||_F^2.
\end{align*}

\section{Omitted Proofs for Section~\ref{sec:proof-sketch-twolayer}} %of Theorem \ref{thm:main-theorem-twolayer}}
\label{sec:twolayerformal}
%\haoyu{Pass through this section. I add a restatement of Lemma 1 and fix some overfull boxes.}
In this section, we will give a formal proof of Theorem \ref{thm:main-theorem-twolayer}. We will follow the proof sketch in Section \ref{sec:proof-sketch-twolayer}. First in Section~\ref{subsec:landscapeproof} we prove Lemma~\ref{lem:geo-property} which gives the optimization landscape for the two-layer neural network with \emph{large enough} width;then in Section~\ref{subsec:2layeralg} we will show that the training process on the function with regularization will end in polynomial time.

\subsection{Optimization landscape of two-layer neural net}\label{subsec:landscapeproof}
In this part we will prove the optimization landscape(Lemma \ref{lem:geo-property}) of 2-layer neural network. First we recall Lemma \ref{lem:geo-property}.

\lemoptlandscape*

For simplicity, we will use $\delta_j(W) = \sum_{i=1}^r a_i(w_i^T x_j)^2 - y_j$ to denote the error of the output of the neural network and the label $y_j$. Consider the matrix $M = \frac{1}{n}\sum_{j=1}^n\delta_j x_jx_j^T$. To show that every $\varepsilon$-second-order stationary point $W$ of $f$ will have small function value $f(W)$, we need the following 2 lemmas.

Generally speaking, the first lemma shows that, when the network is \emph{large enough}, any point with \emph{almost Semi-definite Hessian} will lead to a small spectral norm of matrix $M$. %The lemma and the proof are shown below.

\lemsmallesteigenvalue*

\begin{proof}
    First note that the equation
    \[\lambda_{\min}\nabla^2 f(W) = -\max_i |\lambda_i (M)|\]
    is equivalent to
    \[\min_{||Z||_F = 1}\nabla^2 f(W)(Z,Z) = -\max_{||z||_2 = 1} |z^T Mz|,\]
    and we will give a proof of the equivalent form.
    
    First, we show that
    \[\min_{||Z||_F = 1}\nabla^2 f(W)(Z,Z) \ge -\max_{||z||_2 = 1} |z^T Mz|.\]
    Intuitively, this is because $\nabla^2 f(W)$ is the sum of two terms, one of them is always positive semidefinite, and the other term is equivalent to a weighted combination of the matrix $M$ applied to different columns of $Z$.
    \begin{align*}
        &\nabla^2 f(W)(Z,Z)\\
        =& \sum_{k=1}^r z_k^T\left(\frac{a_{k}}{n}\sum_{j=1}^n \left(\sum_{i=1}^r a_i(w_i^Tx_j)^2 - y_j\right)x_jx_j^T\right)z_k + \frac{2}{n}\sum_{j=1}^n\left(\sum_{i=1}^r a_i w_i^Tx_jx_j^Tz_i\right)^2 \\
        =& \sum_{k=1}^r a_k z_k^T M z_k + \frac{2}{n}\sum_{j=1}^n\left(\sum_{i=1}^r a_i w_i^Tx_jx_j^Tz_i\right)^2 \\
        \ge & \sum_{k=1}^r a_k z_k^T M z_k \\
        \ge & -\sum_{k=1}^r \max_i|\lambda_{i}(M)|\cdot ||z_k||_2^2 \\
        =& -\max_i|\lambda_{i}(M)|\cdot ||Z||_F^2.
    \end{align*}
    Then we have
    \[\min_{||Z||_F = 1}\nabla^2 f(W)(Z,Z) \ge \min_{||Z||_F = 1}(-\max_i|\lambda_{i}(M)|\cdot ||Z||_F^2) = -\max_i|\lambda_{i}M| = -\max_{||z||_2 = 1} |z^T Mz|.\]
    
    For the other side, we show that
    \[\min_{||Z||_F = 1}\nabla^2 f(W)(Z,Z) \le -\max_{||z||_2 = 1} |z^T Mz|\]
    by showing that there exists $Z,||Z||_F = 1$ such that $\nabla^2 f(W)(Z,Z) = -\max_{||z||_2 = 1} |z^T Mz|$. 
    
    First, let $z_0 = \arg\max_{||z||_2 = 1} |z^T Mz|$. Recall that for simplicity, we assume that $r$ is an even number and $a_i = 1$ for all $i \le \frac{r}{2}$ and $a_i = -1$ for all $i \ge \frac{r+2}{2}$. If $z_0^T Mz_0 < 0$, there exists $u\in\R^{r}$ such that
    \begin{enumerate}
        \item $||u||_2 = 1$,
        \item $u_i = 0$ for all $i \ge \frac{r+2}{2}$,
        \item $\sum_{i=1}^r a_i u_i w_i = \textbf{0}$,
    \end{enumerate}
    since for constraints 2 and 3, they form a homogeneous linear system, and constraint 2 has $\frac{r}{2}$ equations and constraint 3 has $d$ equations. The total number of the variables is $r$ and we have $r > \frac{r}{2} + d$ since we assume that $r \ge 2d+2$. Then there must exists $r\neq \textbf{0}$ that satisfies constraints 2 and 3. Then we normalize that $u$ to have norm $||u||_2 = 1$.
    
    Then, let $Z = z_0u^T$, we have $||Z||_F^2 = ||z_0||_2^2\cdot ||u||_2^2 = 1$ and
    \begin{align*}
        \nabla^2 f(W)(Z,Z) =& \sum_{k=1}^r a_k z_k^T M z_k + \frac{2}{n}\sum_{j=1}^n\left(\sum_{i=1}^r a_i w_i^Tx_jx_j^Tz_i\right)^2\\
        =& \sum_{k=1}^r a_k u_k^2 z_0^T M z_0 + \frac{2}{n}\sum_{j=1}^n\left(\sum_{i=1}^r a_i u_i w_i^Tx_jx_j^Tz_0\right)^2\\
        =& z_0^TMz_0 + \frac{2}{n}\sum_{j=1}^n\left(\sum_{i=1}^r\textbf{0}^Tx_jx_j^Tz_0\right)^2\\
        =& -\max_{||z||_2 = 1} |z^T Mz|,
    \end{align*}
    where the third equality comes from the fact that $||u||_2^2 = \sum_{i=1}^r u_i^2 = 1$, $u_i = 0$ for all $i > \frac{r}{2}$, and $\sum_{i=1}^r a_i u_i w_i = \textbf{0}$. The proof for the case when $z_0^T Mz_0 > 0$ is symmetric, except we use the second half of the coordinates (where $a_i = -1$).
    
    \iffalse
    Similarly, if $z_0^T Mz_0 > 0$, there exists $u\in\R^{r}$ such that
    \begin{enumerate}
        \item $||u||_2 = 1$,
        \item $u_i = 0$ for all $i \le \frac{r}{2}$,
        \item $\sum_{i=1}^r a_i u_i w_i = \textbf{0}$.
    \end{enumerate}
    Let $Z = z_0u^T$ and $||Z||_F = 1$, we have
    \begin{align*}
        \nabla^2 f(W)(Z,Z) =& \sum_{k=1}^r a_k z_k^T M z_k + \frac{2}{n}\sum_{j=1}^n\left(\sum_{i=1}^r a_i w_i^Tx_jx_j^Tz_i\right)^2\\
        =& \sum_{k=1}^r a_k u_k^2 z_0^T M z_0 + \frac{2}{n}\sum_{j=1}^n\left(\sum_{i=1}^r a_i u_i w_i^Tx_jx_j^Tz_0\right)^2\\
        =& z_0^TMz_0 + \frac{2}{n}\sum_{j=1}^n\left(\sum_{i=1}^r\textbf{0}^Tx_jx_j^Tz_0\right)^2\\
        =& -\max_{||z||_2 = 1} |z^T Mz|.
    \end{align*}
    Then, we just prove that
    \[\min_{||Z||_F = 1}\nabla^2 f(W)(Z,Z) \le -\max_{||z||_2 = 1} |z^T Mz|\]
    and we conclude the proof by combining the previous result.
    \fi
\end{proof}

The next step needs to connect the matrix $M$ and the loss function. In particular, we will show that if the spectral norm of $M$ is small, the loss is also small. 
%Besides the previous lemma, the next lemma shows that, if the samples are in \emph{general position} when the spectral norm of matrix $M$ is small, we know that the function value $f(W)$ is also small. The lemma and the proof are shown below.

\lemspectralnormandvalue*

\begin{proof}
    We know that the function value $f(W) = \frac{1}{n}\sum_{j=1}^n \delta_j^2 = \frac{1}{n}||\delta||_2^2$, where $\delta \in\R^n$ is the vector whose $j$-th element is $\delta_j$. Because $X = [x_1^{\otimes 2},\dots,x_n^{\otimes 2}]\in \R^{d^2 \times n}$ has full column rank and the smallest singular value is at least $\sigma$, we know that for any $v\in\R^n$,
    \[||Xv||_2 \ge \sigma_{\min}(X)\cdot ||v||_2 \ge \sigma ||v||_2.\]
    Since $M = \frac{1}{n}\sum_{j=1}^n \delta_j x_jx_j^T$ is a symmetric matrix, $M$ has $d$ real eigenvalues, and we use $\lambda_1,\dots,\lambda_d$ to denote these eigenvalues. Because we assume that the spectral norm of the matrix $M = \frac{1}{n}\sum_{j=1}^n \delta_j x_jx_j^T$ is upper bounded by $\lambda$, which means that $|\lambda_i| \le \lambda$ for all $1\le i\le d$, and we have
    \[||M||_F^2 = \sum_{i=1}^d \lambda_i^2 \le \sum_{i=1}^d \lambda^2 = d\lambda^2.\]
    Then we can conclude that
    \[||M||_F^2 = ||\frac{1}{n}\sum_{j=1}^n \delta_j x_jx_j^T||_F^2 = \frac{1}{n^2}||X\delta||_2^2 \ge \frac{1}{n^2}\sigma^2 ||\delta||_2^2,\]
    where the second equation comes from the fact that reordering a matrix to a vector preserves the Frobenius norm.
    
    Then combining the previous argument, we have
    \[f(W) = \frac{1}{4n}||\delta||_2^2 \le \frac{n}{4\sigma^2}||M||_F^2 \le \frac{nd\lambda^2}{4\sigma^2}.\]
\end{proof}

Lemma~\ref{lem:geo-property} follows immediately from Lemma~\ref{lem:smallesteigenvalue} and Lemma~\ref{lem:spectralnormandfuncvalue}.
%Combining these 2 lemmas together, it is easy to complete the proof of Lemma \ref{lem:geo-property}.

\subsection{Training guarantee of the two-layer neural net}
\label{subsec:2layeralg}
Recall that in order to derive the time complexity for the training procedure, we add a regularizer to the function $f$. More concretely, 
\[g(W) = f(W) + \frac{\gamma}{2}||W||_F^2,\]
where $\gamma$ is a constant that we choose in Theorem~\ref{thm:main-theorem-twolayer}.

To analyze the running time of the PGD algorithm, we first bound the smoothness and Hessian Lipschitz parameters when the Frobenius norm of $W$ is bounded.

\begin{restatable}{lemma}{lemsmoothness}\label{lem:smoothness-lipschitz-hessian}
    In the set $\{W:||W||_F^2 \le \Gamma\}$, if we have $||x_j||_2 \le B$ and $|y_j| \le Y$ for all $j \le n$, then
    \begin{enumerate}
        \item $\nabla g(W)$ is $(3B^4\Gamma+YB^2+\gamma)$-smooth.
        \item $\nabla^2 g(W)$ has $6B^4\Gamma^{\frac{1}{2}}$-Lipschitz Hessian.
    \end{enumerate}
\end{restatable}

\begin{proof}
    We first figure out the smoothness constant. We have
    \begin{align*}
        &||\nabla g(U) - \nabla g(V)||_F \\ 
        =& ||\frac{1}{n}\sum_{j=1}^n \left(x_j^TUAU^Tx_j-y_j\right)x_jx_j^TUA + \gamma U - \frac{1}{n}\sum_{j=1}^n \left(x_j^TVAV^Tx_j-y_j\right)x_jx_j^TVA - \gamma V||_F \\
        \le& ||\frac{1}{n}\sum_{j=1}^n \left(x_j^TUAU^Tx_j-y_j\right)x_jx_j^TUA - \frac{1}{n}\sum_{j=1}^n \left(x_j^TVAV^Tx_j-y_j\right)x_jx_j^TVA||_F + \gamma ||U-V||_F.
    \end{align*}
    Then we bound the first term, we have
    \begin{align*}
        & ||\frac{1}{n}\sum_{j=1}^n \left(x_j^TUAU^Tx_j-y_j\right)x_jx_j^TUA - \frac{1}{n}\sum_{j=1}^n \left(x_j^TVAV^Tx_j-y_j\right)x_jx_j^TVA||_F \\
        =& ||\frac{1}{n}\sum_{j=1}^n \left(x_j^TUAU^Tx_j-y_j\right)x_jx_j^TUA - \frac{1}{n}\sum_{j=1}^n \left(x_j^TUAU^Tx_j-y_j\right)x_jx_j^TVA \\
        &\quad + \frac{1}{n}\sum_{j=1}^n \left(x_j^TUAU^Tx_j-y_j\right)x_jx_j^TVA - \frac{1}{n}\sum_{j=1}^n \left(x_j^TUAV^Tx_j-y_j\right)x_jx_j^TVA \\
        &\quad + \frac{1}{n}\sum_{j=1}^n \left(x_j^TUAV^Tx_j-y_j\right)x_jx_j^TVA - \frac{1}{n}\sum_{j=1}^n \left(x_j^TVAV^Tx_j-y_j\right)x_jx_j^TVA||_F\\
        \le& ||\frac{1}{n}\sum_{j=1}^n \left(x_j^TUAU^Tx_j-y_j\right)x_jx_j^TUA - \frac{1}{n}\sum_{j=1}^n \left(x_j^TUAU^Tx_j-y_j\right)x_jx_j^TVA||_F \\
        &\quad + ||\frac{1}{n}\sum_{j=1}^n \left(x_j^TUAU^Tx_j-y_j\right)x_jx_j^TVA - \frac{1}{n}\sum_{j=1}^n \left(x_j^TUAV^Tx_j-y_j\right)x_jx_j^TVA||_F \\
        &\quad + ||\frac{1}{n}\sum_{j=1}^n \left(x_j^TUAV^Tx_j-y_j\right)x_jx_j^TVA - \frac{1}{n}\sum_{j=1}^n \left(x_j^TVAV^Tx_j-y_j\right)x_jx_j^TVA||_F.\\
    \end{align*}
    The first term can be bounded by
    \begin{align*}
        &||\frac{1}{n}\sum_{j=1}^n \left(x_j^TUAU^Tx_j-y_j\right)x_jx_j^TUA - \frac{1}{n}\sum_{j=1}^n \left(x_j^TUAU^Tx_j-y_j\right)x_jx_j^TVA||_F \\
        \le& ||\frac{1}{n}\sum_{j=1}^n \left(x_j^TUAU^Tx_j\right)x_jx_j^TUA - \frac{1}{n}\sum_{j=1}^n \left(x_j^TUAU^Tx_j\right)x_jx_j^TVA||_F\\
        &\quad+ ||\frac{1}{n}\sum_{j=1}^n y_jx_jx_j^TUA - y_jx_jx_j^TVA||_F \\
        \le& ||\frac{1}{n}\sum_{j=1}^n \left(x_j^TUAU^Tx_j\right)x_jx_j^T||_F||(U-V)A||_F + YB^2||(U-V)A||_F \\
        \le& B^4\Gamma ||U-V||_F + YB^2||U-V||_F.
    \end{align*}
    Similarly, we can show that
    \[||\frac{1}{n}\sum_{j=1}^n \left(x_j^TUAU^Tx_j-y_j\right)x_jx_j^TVA - \frac{1}{n}\sum_{j=1}^n \left(x_j^TUAV^Tx_j-y_j\right)x_jx_j^TVA||_F \le B^4\Gamma ||U-V||_F,\]
    and
    \[||\frac{1}{n}\sum_{j=1}^n \left(x_j^TUAV^Tx_j-y_j\right)x_jx_j^TVA - \frac{1}{n}\sum_{j=1}^n \left(x_j^TVAV^Tx_j-y_j\right)x_jx_j^TVA||_F \le B^4\Gamma ||U-V||_F.\]
    Then, we have
    \[||\nabla g(U) - \nabla g(V)||_F \le (3B^4\Gamma+YB^2+\gamma) ||U-V||_F.\]
    Then we bound the Hessian Lipschitz constant. We have
    \begin{align*}
        &|\nabla^2 g(U)(Z,Z) - \nabla^2 g(V)(Z,Z)|\\
        =& |\sum_{k=1}^r z_k^T\left(\frac{a_{k}}{n}\sum_{j=1}^n \left(x_j^TUAU^Tx_j-y_j\right)x_jx_j^T\right)z_k + \frac{2}{n}\sum_{j=1}^n\left(\sum_{i=1}^r a_i u_i^Tx_jx_j^Tz_i\right)^2 + \gamma ||Z||_F^2 \\
        &\quad - \sum_{k=1}^r z_k^T\left(\frac{a_{k}}{n}\sum_{j=1}^n \left(x_j^TVAV^Tx_j-y_j\right)x_jx_j^T\right)z_k - \frac{2}{n}\sum_{j=1}^n\left(\sum_{i=1}^r a_i v_i^Tx_jx_j^Tz_i\right)^2 - \gamma ||Z||_F^2| \\
        \le&\sum_{k=1}^r|z_k^T\left(\frac{a_{k}}{n}\sum_{j=1}^n \left(x_j^T(UAU^T-VAV^T)x_j\right)x_jx_j^T\right)z_k|\\
        &\quad+ \frac{2}{n}\sum_{j=1}^n|\left(\sum_{i=1}^r a_i u_i^Tx_jx_j^Tz_i\right)^2 - \left(\sum_{i=1}^r a_i v_i^Tx_jx_j^Tz_i\right)^2|.
    \end{align*}
    First we have
    \begin{align*}
        &|z_k^T\left(\frac{a_{k}}{n}\sum_{j=1}^n \left(x_j^T(UAU^T-VAV^T)x_j\right)x_jx_j^T\right)z_k|\\
        \le& \frac{1}{n}\sum_{j=1}^n||\left(x_j^T(UAU^T-VAV^T)x_j\right)x_jx_j^T||_F||z_k||_2^2 \\
        \le& \frac{1}{n}\sum_{j=1}^n||\left(x_j^T(UAU^T-UAV^T + UAV^T - VAV^T)x_j\right)x_jx_j^T||_F||z_k||_2^2 \\
        \le& 2B^4\Gamma^{\frac{1}{2}}||U-V||_F||z_k||_2^2,
    \end{align*}
    So we can bound the first term by
    \begin{align*}
        &\sum_{k=1}^r|z_k^T\left(\frac{a_{k}}{n}\sum_{j=1}^n \left(x_j^T(UAU^T-VAV^T)x_j\right)x_jx_j^T\right)z_k|\\
        \le&\sum_{k=1}^r2B^4\Gamma^{\frac{1}{2}}||U-V||_F||z_k||_2^2 = 2B^4\Gamma^{\frac{1}{2}}||U-V||_F||Z||_F^2.
    \end{align*}
    Then for the second term, note that
    \[\sum_{i=1}^r a_i u_i^Tx_jx_j^Tz_i = \langle UA, x_jx_j^TZ\rangle,\]
    and we have
    \begin{align*}
        &\frac{2}{n}\sum_{j=1}^n|\left(\sum_{i=1}^r a_i u_i^Tx_jx_j^Tz_i\right)^2 - \left(\sum_{i=1}^r a_i v_i^Tx_jx_j^Tz_i\right)^2| \\
        =& \frac{2}{n}\sum_{j=1}^n |\langle UA, x_jx_j^TZ\rangle^2 - \langle VA, x_jx_j^TZ\rangle^2| \\
        =& \frac{2}{n}\sum_{j=1}^n |\langle (U-V)A, x_jx_j^TZ\rangle\langle (U+V)A, x_jx_j^TZ\rangle| \\
        \le& \frac{2}{n}\sum_{j=1}^n||(U-V)A||_F||x_jx_j^TZ||_F||(U+V)A||_F||x_jx_j^TZ||_F \\
        \le& 4B^4\Gamma^{\frac{1}{2}}||U-V||_F||Z||_F^2,
    \end{align*}
    where the first inequality comes from the Cauchy-Schwatz inequality. Combining with the previous computation, we have
    \[|\nabla^2 g(U)(Z,Z) - \nabla^2 g(V)(Z,Z)| \le 6B^4\Gamma^{\frac{1}{2}}||U-V||_F||Z||_F^2.\]
\end{proof}

We also have the theorem showing the convergence result of Perturbed Gradient Descent(Algorithm \ref{alg:pgd}).

\thmpgdconvergence*

Then based on the convergence result in \cite{jin2017escape} and the previous lemmas, we have the following main theorem for 2-layer neural network with quadratic activation.

\thmmainthmtwolayer*

\begin{proof}[Proof of Theorem \ref{thm:main-theorem-twolayer}]
    We first show that during the training process, if the constant $c \le 1$, the objective function value satisfies
    \[g(W_t) \le g(W_{\text{ins}}) + \frac{3c\varepsilon^2}{2\chi^4},\]
    where we choose the smoothness constant $\ell \ge 1$ to be the smoothness for the region $g(W) \le  g(W_{\text{ins}}) + \frac{3c\varepsilon^2}{2\chi^4}$.
    
    In the PGD algorithm (Algorithm~\ref{alg:pgd}), we say a point is in a perturbation phase, if $t-t_{noise} < t_{thres}$. A point $x_t$ is the beginning of a perturbation phase if it reaches line 5 of Algorithm~\ref{alg:pgd} and a perturbation is added to it.
    
    We use induction to show that the following properties hold. % with high probability.
    \begin{enumerate}
        \item If time $t$ is not in the perturbation phase, then $g(W_t) \le g(W_{\text{ins}})$.
        \item If time $t$ is in a perturbation phase, then $g(W_t) \le g(W_{\text{ins}}) + \frac{3c\varepsilon^2}{2\chi^4\ell}$. Moreover, if $t$ is the beginning of a perturbation phase, then $g(\tilde W_t) \le g(W_{\text{ins}})$.
    \end{enumerate}
    First we show that at time $t=0$, the property holds. If $t=0$ is not the beginning of a perturbation phase, then the inequality holds trivially by initialization.
    If $t=0$ is the beginning of a perturbation phase, then we know that $g(\tilde W_0) = g(W_{\text{ins}})$ from the definition of the algorithm, then
    \begin{align}
        g(W_0) =& g(\tilde W_0 + \xi_0)\label{equ:perturb} \\
        \le& g(\tilde W_0) + ||\xi_0||_F||\nabla g(\tilde W_0)||_F + \frac{\ell}{2}||\nabla g(\tilde W_0)||_F^2\nonumber \\
        \le& g(\tilde W_0) + r\cdot g_{\text{thres}} + \frac{\ell}{2}r^2\nonumber \\
        \le& g(\tilde W_0) + \frac{\sqrt{c}\varepsilon}{\chi^2\ell} \cdot \frac{\sqrt{c}\varepsilon}{\chi^2} +\frac{\ell}{2} \frac{\sqrt{c}\varepsilon}{\chi^2\ell}\cdot \frac{\sqrt{c}\varepsilon}{\chi^2\ell}\nonumber \\
        =& g(W_{\text{ins}}) + \frac{3c\varepsilon^2}{2\chi^4\ell}.\nonumber
    \end{align}
    
    Now we do the induction: assuming the two properties hold for time $t$, we will show that they also hold at time $t+1$. We break the proof into 3 cases:
    
    %Then we show that if at time $t$ the induction property is satisfied, then the induction property holds at time $t+1$.
    {\bf Case 1}: $t+1$ is not in a perturbation phase. In this case, the algorithm does not add a perturbation on $W_{t+1}$, and we have
    
    \begin{align}
        g(W_{t+1}) =& g(W_{t} - \eta \nabla g(W_{t}))\label{equ:smoothness} \\ 
        \le& g(W_{t}) - \langle \eta \nabla g(W_{t}), \nabla g(W_t)\rangle + \frac{\ell}{2}||\eta \nabla g(W_t)||_F^2\nonumber \\
        \le& g(W_t) - \frac{\eta}{2}||\nabla g(W_t)||_F^2||\nonumber\\
        \le& g(W_t). \nonumber
    \end{align}
    
    If $t$ is not in a perturbation phase, then from the induction hypothesis, we have
    \[g(W_{t+1}) \le g(W_t) \le g(W_{\text{ins}}),\]
    otherwise if $t$ is in a perturbation phase, since $t+1$ is not in a perturbation phase, $t$ must be at the end of the phase. By design of the algorithm we have:
    \[g(W_{t+1}) \le g(W_t) \le g(\tilde W_{t_{\text{noise}}}) - f_{\text{thres}} \le g(W_{\text{ins}}).\]
    
    {\bf Case 2}: $t+1$ is in a perturbation phase, but not at the beginning.
    %Similarly, if $t+1$ is in a perturbation phase. If $t+1$ is not the start of a perturbation phase, and from the same computation as \ref{equ:smoothness}, we have
    Using the same reasoning as \eqref{equ:smoothness}, we know
    \[g(W_{t+1}) \le g(W_t) \le g(W_{\text{ins}}).\]
    {\bf Case 3}: $t+1$ is at the beginning of a perturbation phase. First we know that
    \[g(W_{t}) \le g(W_{\text{ins}}),\]
    since $t$ is either not in a perturbation phase of at the end of a perturbation phase, then we have $g(\tilde W_{t+1}) \le g(W_{\text{ins}})$. Same as the computation in \eqref{equ:perturb}, we have
    \[g(W_{t+1}) \le g(W_{\text{ins}}) + \frac{3c\varepsilon^2}{2\chi^4\ell}.\]
    This finishes the induction.
    
    Since we choose $\ell \ge 1$, we can choose the other parameters such that $g(W_{t+1}) \le g(W_{\text{ins}}) + \frac{3c\varepsilon^2}{2\chi^4} \le g(W_{\text{ins}}) + 1$. Then since
    \[g(W) = f(W) + \frac{\gamma}{2}||W||_F^2,\]
    we know that during the training process, we have $||W||_F^2 \le \frac{2(g(W_{\text{ins}}) + 1)}{\gamma}$. Since we train from $W_{\text{ins}} = 0$, we have $||W||_F^2 \le \frac{2(f(0) + 1)}{\gamma}$. From Lemma \ref{lem:smoothness-lipschitz-hessian}, we know that 
    \begin{enumerate}
        \item $\nabla g(W)$ is $(3B^4\frac{2(f(0) + 1)}{\gamma}+YB^2+\gamma)$-smooth.
        \item $\nabla^2 g(W)$ has $6B^4\sqrt{\frac{2(f(0) + 1)}{\gamma}}$-Lipschitz Hessian.
    \end{enumerate}
    As we choose $\gamma = (6B^4\sqrt{2(f(0) + 1)})^{2/5}\cdot \varepsilon^{2/5}$, we know that $\rho = (6B^4\sqrt{2(f(0) + 1)})^{4/5}\cdot \varepsilon^{-1/5}$ is an upper bound on the Lipschitz Hessian constant.
    
    When PGD stops, we know that
    \[\lambda_{\min}(\nabla^2 g(W)) \ge -\sqrt{\rho\varepsilon} = -(6B^4\sqrt{2(f(0) + 1)})^{2/5}\cdot \varepsilon^{2/5},\]
    and we have
    \[\lambda_{\min}(\nabla^2 f(W)) \ge \lambda_{\min}(\nabla^2 g(W)) - \gamma \ge -2(6B^4\sqrt{2(f(0) + 1)})^{2/5}\cdot \varepsilon^{2/5}.\]
    From Lemma \ref{lem:smallesteigenvalue}, we know that the spectral norm of matrix $M$ is bounded by $2(6B^4\sqrt{2(f(0) + 1)})^{2/5}\cdot \varepsilon^{2/5}$, and from Lemma \ref{lem:spectralnormandfuncvalue}, we know that
    \[f(W) \le \frac{nd\cdot 4(6B^4\sqrt{2(f(0) + 1)})^{4/5}\cdot \varepsilon^{4/5}}{4\sigma^2} = \frac{nd\cdot (6B^4\sqrt{2(f(0) + 1)})^{4/5}\cdot \varepsilon^{4/5}}{\sigma^2}. \]
    The running time follows directly from the convergence theorem of Perturbed Gradient Descent(Theorem \ref{thm:pgd-convergence}) and the previous argument that the training trajectory will not escape from the set $\{W:||W||_F^2 \le \frac{2(g(W_{\text{ins}}) + 1)}{\gamma}\}$.
    
    Then, in order to get the error to be smaller than $\varepsilon$, we choose
    \[\varepsilon' = \left(\frac{\sigma^2\varepsilon}{nd}\right)^{5/4}\frac{1}{6B^4\sqrt{2(f(0) + 1)}},\]
    and the total running time should be
    \[O\left(\frac{B^8\ell(nd)^{5/2}(f(0)+1)^2}{\sigma^{5}\varepsilon^{5/2}}\log^4\left(\frac{Bnrd\ell\Delta(f(0)+1)}{\varepsilon^2\delta\sigma}\right)\right).\]
    Besides, our parameter $\rho$ and $\gamma$ is chosen to be
    \[\rho = (6B^4\sqrt{2(f(0) + 1)})^{4/5}\cdot \varepsilon'^{-1/5} = (6B^4\sqrt{2(f(0) + 1)})\left(\frac{nd}{\sigma^2\varepsilon}\right)^{\frac{1}{4}},\]
    and
    \[\gamma = (6B^4\sqrt{2(f(0) + 1)})^{2/5}\cdot \varepsilon'^{2/5} = \left(\frac{\sigma^2\varepsilon}{nd}\right)^{\frac{1}{2}}.\]
\end{proof}

\section{Omitted Proofs in Section~\ref{sec:proof-sketch-random-feature}}%  Proof of Theorem \ref{thm:main-theorem-random-feature-with-perturbation} and \ref{thm:main-theorem-random-feature}}
\label{sec:threelayerformal}
%\haoyu{Add the following summary of the whole section, need polish.}
In this section, we give the proof of the main results of our three-layer neural network(Theorem \ref{thm:main-theorem-random-feature-with-perturbation} and \ref{thm:deterministic}). Our proof mostly uses leave-one-out distance to bound the smallest singular value of the relevant matrices, which is a common approach in random matrix theory (e.g., in \cite{}). However, the matrices we are interested in involves high order tensor powers that have many correlated entries, so we need to rely on tools such as %decoupling and 
anti-concentration for polynomials in order to bound the leave-one-out distance.

First in Section \ref{subsec:pre-random-feature}, we introduce some more notations and definitions, and present some well-known results that will help us present the proofs. In Section \ref{subsec:proof-random-feature-perturb}, we proof Theorem \ref{thm:main-theorem-random-feature-with-perturbation} which focus on the smoothed analysis setting. Finally in Section \ref{subsec:proof-deterministic} we prove Theorem \ref{thm:deterministic} where we can give a deterministic condition for the input.

\subsection{Preliminaries}\label{subsec:pre-random-feature}
%\haoyu{Finish the pass on this subsection. }
%Let $\N$ denote the set of natural numbers (non-negative integers).
\paragraph{Representations of symmetric tensors}
Throughout this section, we use $T_d^p$ to denote the space of $p$-th order tensors on $d$ dimensions. That is, $T_d^p = (\R^d)^{\otimes p}$. A tensor $T \in T_d^p$ is symmetric if $T(i_1,i_2,...,i_p) = T(i_{\pi(1)},i_{\pi(2)}, ..., i_{\pi(p)}$ for every permutation $\pi$ from $[p]\to[p]$. We use $X_d^p$ to denote the space of all symmetric tensors in $T_d^p$. %\haoyu{I do not know how to give a good definition of the symmetric tensors.} 
The dimension of $X_d^p$ is $D_d^p=\binom{p+d-1}{p}$. 

Let $\bar{X}_d^p=\left\{x\in X_d^p\Big|\Vert x\Vert_2=1\right\}$ be the set of unit tensors in $X^p_d$ (as a sub-metric space of $T^p_d$). For $\R^d$, let $\{e_i|i=1,2\cdots d\}$ be its standard orthonormal basis. For simplicity of notation we use $S_p$ to denote the group of bijections (permutations) $[p]\to[p]$, and $I^p_d$ to denote the set of integer indices $I^p_d=\{(i_1,i_2\cdots i_d)\in \N^d|\sum\limits_{j=1}^d i_j=p\}$. We can make $X_d^p$ isomorphic (as a vector space over $\R$) to Euclidean space $\R^{I^p_d}$ ($|I^p_d|=D_d^p$) by choosing a basis $\{s_{(i_1,i_2\cdots i_d)\in I^p_d}=\frac{1}{\prod\limits_{j=1}^d{i_j!}}\sum\limits_{\sigma\in S_p}e_{j_{\sigma(1)}}\otimes e_{j_{\sigma(2)}}\otimes\cdots\otimes e_{j_{\sigma(p)}}|(j_1\circ j_2\circ \cdots \circ j_p)=(1^{(i_1)}\circ2^{(i_2)}\circ\cdots \circ d^{(i_d)})\}$ where $(1^{(i_1)}\circ2^{(i_2)}\circ\cdots \circ d^{(i_d)})$ means a length $p$ string with $i_1$ 1's, $i_2$ 2's and so on, and let the isomorphism be $\phi^p_d$. We call the image of a symmetric tensor through $\phi^p_d$ its \emph{reduced vectorized form}, and we can define a new norm on $X_d^p$ with $\Vert x\Vert_{\text{rv}}=\Vert \phi^p_d(x)\Vert_2$.

Given the definition of \emph{reduced vectorized form} and the norm $\Vert\cdot\Vert_{\text{rv}}$, we have the following lemma that bridges between the norm $\Vert\cdot\Vert_{\text{rv}}$ and the original 2-norm.

\begin{lemma}\label{lem:rv-metric}
        For any $x\in X_n^p$,
        \begin{equation*}
            \Vert x\Vert_{\text{rv}} \geq \frac{1}{\sqrt{p!}}\Vert x\Vert_2.
        \end{equation*}
    \end{lemma}
    \begin{proof}
        We can expand $x$ as $x=\sum\limits_{i\in I_n^p} x_i s_i$. Then $\Vert x\Vert_{\text{rv}}=\sqrt{\sum\limits_{i\in I_n^p} x_i^2}$ and $\Vert x\Vert_2=\sqrt{\sum\limits_{i\in I_n^p} x_i^2 \Vert s_i\Vert_2^2}$ as $\{s_i\}$ are orthogonal. Notice that for $i=(i_1,i_2\cdots i_n)$, $\Vert s_i\Vert_2^2 = \frac{p!}{\prod\limits_{j=1}^n i_j!}\leq p!$, and therefore
        \[\Vert x\Vert_{2} \leq \sqrt{\sum\limits_{i\in I_n^p} x_i^2 p!} =\sqrt{p!} \Vert x\Vert_{\text{rv}}.\]
    \end{proof}

\paragraph{$\varepsilon$-net} Part of our proof uses $\varepsilon$-nets to do a covering argument. Here we give its definition.
%Since we use the $\varepsilon$-net argument in the proof of Theorem \ref{thm:main-theorem-random-feature}, we will give the definition of $\varepsilon$-net and the proposition that bounds the size of an $\varepsilon$-net for the symmetric tensors.

\begin{definition}[$\varepsilon$-Net]
    Given a metric space $(X,d)$. A finite set $N\subseteq \cP$ is called an $\varepsilon$-net for $\cP\subset X$ if for every $\bx\in\cP$, there exists $\pi(\bx)\in N$ such that $d(\bx, \pi(\bx))\le\varepsilon$. The smallest cardinality of an $\varepsilon$-net for $\cP$ is called the covering number:
    $\cN(\cP,\varepsilon) = \inf\{|N|:N \text{ is an $\varepsilon$-net of $\cP$}\}$.
\end{definition}

%\begin{remark}
%    Note that in our definition of $\varepsilon$-net, we require that the $\varepsilon$-net of $\cP$ be a subset of $\cP$ itself. However, there exists another definition that does not require this restraint. We adopt this definition because in our proof, we need the $\varepsilon$-net to be a subset of the original set.
%\end{remark}

Then we introduce give an upper bound on the size of $\varepsilon$-net of a set $K\subseteq \cR^d$. First, we need the definition of Minkowski sum
\begin{definition}[Minkowski sum]
    Let $A,B\subseteq \cR^d$ be 2 subsets of $\cR^d$, then the Minkowski sum is defined as
    \[A + B := \{a+b:a\in A,b\in B\}.\]
\end{definition}

Then the covering number can be bounded by a volume argument. This is well-known, and the proof can be found in \cite{vershynin2018high}(Proposition 4.2.12 in \cite{vershynin2018high}).

\begin{proposition}[Covering number]\label{prop:covering-number}
    Given a set $K\subseteq \cR^d$ and the corresponding metric $d(x,y) := \Vert x-y\Vert_2$. Suppose that $\varepsilon > 0$, and then we have
    \[\cN(K,\varepsilon)\le \frac{|K+\mathbb B_2^d(\varepsilon / 2)|}{|\mathbb B_2^d(\varepsilon / 2)|},\]
    where $|\cdot|$ denote the volume of the set.
\end{proposition}

Then with the help of the previous proposition, we can now bound the covering number of symmetric tensors with unit length.

\begin{lemma}[Covering number of $\bar{X}_d^p$]\label{lem:eps-net}
    There exists an $\varepsilon$-net of $\bar{X}_d^p$ with size $O\left(\left(1+\frac{2\sqrt{p!}}{\varepsilon}\right)^{D^p_d}\right)$, i.e.
    \[\cN(\bar{X}_d^p,\varepsilon) \le O\left(\left(1+\frac{2\sqrt{p!}}{\varepsilon}\right)^{D^p_d}\right).\]
\end{lemma}

\begin{proof}
    Recall that $\phi^p_d(\cdot):\R^{d^p}\to\R^{D^p_d}$ is an bijection between the symmetric tensors in $\R^{d^p}$ and a vector in $\R^{D^p_d}$. We first show that an $\frac{\varepsilon}{\sqrt{p!}}$-net for the image $\phi^p_d(\bar{X}_d^p)$ implies an $\varepsilon$-net for the unit symmetric tensor $\bar{X}_d^p$.
    
    Suppose that the $\frac{\varepsilon}{\sqrt{p!}}$-net for the image $\phi^p_d(\bar{X}_d^p)$ is denoted as $N\subset \phi^p_d(\bar{X}_d^p)$, and for any $x\in\phi^p_d(\bar{X}_d^p)$, there exists $\pi(x)\in N$ such that $||\pi(x) - x||_2 \le \frac{\varepsilon}{\sqrt{p!}}$. Then we know that $\left(\phi^p_d\right)^{-1}(N)$ is an $\varepsilon$-net for the unit symmetric tensors $\bar{X}_d^p$, because for any $x'\in\bar{X}_d^p$, we have
    \begin{align*}
        \Vert x' - \left(\phi^p_d\right)^{-1}(\pi(\phi^p_d(x'))) \Vert_2 \le& \sqrt{p!}\Vert\phi^p_d(x') - \pi(\phi^p_d(x'))\Vert_2 \\
        \le& \sqrt{p!}\cdot \frac{\varepsilon}{\sqrt{p!}}\\
        =& \varepsilon,
    \end{align*}
    where the first inequality comes from Lemma \ref{lem:rv-metric}.
    
    Next, we bound the covering number for the set $\phi^p_d(\bar{X}_d^p)$. First note that the set satisfies $\phi^p_d(\bar{X}_d^p)\subset\R^{D^p_d}$, and from Proposition \ref{prop:covering-number}, we have
    \begin{align*}
        \cN\left(\phi^p_d(\bar{X}_d^p),\frac{\varepsilon}{\sqrt{p!}}\right) \le& \frac{\bigg|\phi^p_d(\bar{X}_d^p)+\mathbb B_2^{D^p_d}(\frac{\varepsilon}{2\sqrt{p!}})\bigg|}{\bigg|\mathbb B_2^{D^p_d}(\frac{\varepsilon}{2\sqrt{p!}})\bigg|}\\
        \le& \frac{\bigg|\mathbb B_2^{D^p_d}(1)+\mathbb B_2^{D^p_d}(\frac{\varepsilon}{2\sqrt{p!}})\bigg|}{\bigg|\mathbb B_2^{D^p_d}(\frac{\varepsilon}{2\sqrt{p!}})\bigg|}\\
        =&\left(1+\frac{2\sqrt{p!}}{\varepsilon}\right)^{D^p_d},
    \end{align*}
    where the first inequality comes from Proposition \ref{prop:covering-number} and the second inequality comes from the fact that $||\phi^p_d(x)||_2 \le ||x||_2$.
\end{proof}

\paragraph{Leave-one-out Distance} Another main ingredient in our proof is  \emph{Leave-one-out distance}. This is a notion that is closely related to the smallest singular value, but usually much easier to compute and bound. It has been widely used in random matrix theory, for example in \cite{rudelson2009smallest}. % and a corresponding proposition. \emph{Leave-on-out} distance is used for the proof of Theorem \ref{thm:main-theorem-random-feature-with-perturbation}. The \emph{Leave-on-out} distance has a strong connection with the smallest singular value of a matrix, and this result is shown in the proposition.

\begin{definition}[Leave-one-out distance]
    For a set of vectors $V=\{v_1,v_2\cdots v_n\}$, their leave-one-out distance is defined as 
    \[l(V)=\min\limits_{1\leq i\leq n}\inf\limits_{a_1,a_2\cdots a_n\in R}\Vert v_i-\sum\limits_{j\neq i}a_jv_j\Vert_2.\] 
    For a matrix $M$, its leave-one-out distance $l(M)$ is the leave-one-out distance of its columns.
\end{definition}

The leave-one-out distance is connected with the smallest singular value by the following lemma:

\begin{lemma}[Leave-one-out distance and smallest singular value]\label{lem:loo-distance-and-singular}
    For a matrix $M\in R^{m\times n}$ with $m\geq n$, let $l(M)$ denote the leave-one-out distance for the columns of $M$, and $\sigma_{\min}(M)$ denote the smallest singular value of $M$, then
    \begin{equation}\nonumber
        \frac{l(M)}{\sqrt{n}}\leq\sigma_{\min}(M)\leq l(M).
    \end{equation}
\end{lemma}

We give the proof for completeness.

\begin{proof}
    For any $x\in \R^n\backslash \{0\}$, let $r(x)=\mathop{\text{argmax}}\limits_{i\in[n]}|x_i|$, then $|x_{r(x)}|>0$ for $x\neq 0$.
    
    Because $l(M)=\min\limits_{i\in [n]}\inf\limits_{x\in \R^n, x_i=1}\Vert Mx\Vert_2$, we have
    \begin{align*}
        \sigma_{\min}(M)=&\inf\limits_{x\in R^n\backslash 0}\frac{\Vert Mx\Vert_2}{\Vert x\Vert_2}\\
        =&\min\limits_{i\in[n]}\inf\limits_{x\in R^n\backslash 0, r(x)=i}\frac{\Vert M\frac{x}{x_i}\Vert_2}{\Vert \frac{x}{x_i}\Vert_2}\\
        =&\min\limits_{i\in[n]}\inf\limits_{x'\in R^n\backslash 0, x'_i=1}\frac{\Vert Mx'\Vert_2}{\Vert x'\Vert_2}.
    \end{align*}
    Because of the equations $\Vert x'\Vert_2\geq |x'_i|=1$ and $\Vert x'\Vert_2=\sqrt{\sum\limits_{j\in [n]}x_j^2}\leq \sqrt{n}|x'_i|=\sqrt{n}$, we have $\frac{l(M)}{\sqrt{n}}\leq\sigma_{\min}(M)\leq l(M)$.
\end{proof}

\paragraph{Anti-concentration}

%Then we recall the anti-concentration result for polynomials whose variables follow Gaussian distribution.

To make use of the random Gaussian noise added in the smoothed analysis setting, we rely on the following anti-concentration result by \cite{carbery2001distributional}:

\begin{restatable}[Anti-concentration (\cite{carbery2001distributional})]{proposition}{thmanticoncentration}\label{prop:anti-concentration}
    For a multivariate polynomial $f(x)=f(x_1,x_2\cdots x_n)$ of degree $p$, let $x\sim \cN(0,1)^n$ follows the standard normal distribution, and $\Var[f]\geq 1$, then for any $t\in \R$ and $\varepsilon>0$,
    \begin{equation}
    \Pr_x[|f(x)-t|\leq\varepsilon]\leq O(p)\varepsilon^{1/p}
    \end{equation}
\end{restatable}

\paragraph{Gaussian moments}

To apply the anti-concentration result, we need to give lower bound of the variance of a polynomial when the variables follow standard Gaussian distribution $\cN(0,1)$. Next, we will show some definitions, propositions, and lemmas that will help us to give lower bound for variance of polynomials. 

\begin{proposition}[Gaussian moments]
    if $x\sim \cN(0,1)$ is a Gaussian variable, then for $p\in N$, $\E_x[x^{2p}]=\frac{(2p)!}{2^p(p!)}\leq 2^pp!$; $\E_x[x^{2p+1}]=0$.
\end{proposition}

\begin{definition}[Hermite polynomials]
    In this paper, we use the normalized Hermite polynomials, which are univariate polynomials which form an orthogonal polynomial basis under the normal distribution. Specifically, they are defined by the following equality
    \begin{equation*}
        H_n(x)=\frac{(-1)^ne^{\frac{x^2}{2}}}{\sqrt{n!}}\left(\frac{d^n e^{-\frac{x^2}{2}}}{dx^n}\right)
    \end{equation*}

\end{definition}
The Hermite polynomials in the above definition forms a set of orthonormal basis of polynomials in the standard Normal distribution. For a polynomial $f: \R^n\to \R$, let $f(x)=\sum\limits_{i\in I_n^{\leq p}} f^M_i \prod\limits_{j=1}^n x_j^{i_j}$ and $f(x)=\sum\limits_{i\in I_n^{\leq p}} f^H_i \prod\limits_{j=1}^n H_{i_j}(x_j)$ be its expansions in the basis of monomials and Hermite polynomials respectively ($H_k$ is the Hermite polynomial of order $k$). Let the index set $I_n^{\leq p}=\bigcup\limits_{j=0}^p I_n^j$. We have the following propositions. The propositions are well-known and easy to prove. We include the proofs here for completeness.% these propositions for self-containment.\haoyu{How to express this sentence in ordinary English?}
\begin{proposition}
    for $i\in I_n^p$, $f^M_i=\left(\prod\limits_{j=1}^n \frac{1}{\sqrt{i_j!}}\right) f^H_i$\label{pro1}
\end{proposition}
\begin{proof}
    Consider $i=(i_1,i_2\cdots i_n)\in I^p_n$, in the monomial expansion, the coefficient for the monomial $M_i=\prod\limits_{j=1}^n x_j^{i_j}$ is $f^M_i$. In the Hermite expansion, since $H_n(x)$ is an order-$n$ polynomial, if the term 
    $\prod\limits_{j=1}^n H_{i'_j}(x_j)$ contain the monomial $M_i$, there must be $i'_j\geq i_j$, and therefore for $i\in I_n^p$ the only term in the Hermite expansion that contains $M_i$ is 
    $f^H_i\prod\limits_{j=1}^n H_{i_j}(x_j)$ (with $M_i$ as its highest order monomial). The coefficient for $x_j^{i_j}$ in $H_{i_j}(x_j)$ is $\frac{1}{\sqrt{i_j!}}$, and therefore $f^M_i=\left(\prod\limits_{j=1}^n \frac{1}{\sqrt{i_j!}}\right) f^H_i$ 
\end{proof}
\begin{proposition}
    For $x\sim \cN(0,1)^n$, $E_x[f]=f^H_{0^n}$, $E_x[f^2]=\sum\limits_{i\in I_n^{\leq p}}(f^H_i)^2$ ($0^n$ refers to the index $(0,0,0\cdots 0)\in I_n^0$).\label{pro2}
\end{proposition}
\begin{proof}
    Firstly, let $w(x)=\frac{1}{\sqrt{2\pi}}e^{-x^2/2}$ be the PDF of $\cN(0,1)$, then
    \begin{align*}
        \int\limits_{-\infty}^{\infty} H_n(x)w(x)dx & = \frac{(-1)^n}{\sqrt{2\pi n!}} \int\limits_{-\infty}^{\infty}\left[\frac{d^n e^{-x^2/2}}{dx^n}\right]dx\\
        & = \left\{
        \begin{matrix}
            0 & n\geq 1\\
            1 & n=0
        \end{matrix}
        \right.,
    \end{align*}
    as a result of $\frac{d^n e^{-x^2/2}}{dx^n}\rightarrow 0$ when $x\rightarrow \pm\infty$ for $n\geq 0$. Besides,
    \begin{align*}
        \int\limits_{-\infty}^{\infty} H_n(x)H_m(x)w(x)dx & = \delta_{nm}
    \end{align*}    
    for its well-known orthogonality in Guassian distribution (with $\delta_{nm}=\I[n=m]$ as the Kronecker function). Therefore,
    \begin{align*}
        E_x[f] & =\sum\limits_{i\in I_n^{\leq p}}f_i^H \prod\limits_{j\in [n]}\int\limits_{-\infty}^{\infty} H_{i_j}(x_j) w(x_j) dx_j \\
        & =\sum\limits_{i\in I_n^{\leq p}}f_i^H \prod\limits_{j\in [n]}\I[i_j=0]\\
        & =f_{0^n}^H,
    \end{align*}
    \begin{align*}
        E_x[f^2] & =\sum\limits_{i,i'\in I_n^{\leq p}}f_i^Hf_{i'}^H \prod\limits_{j\in [n]}\int\limits_{-\infty}^{\infty} H_{i_j}(x_j)H_{i'_j}(x_j) w(x_j) dx_j \\
        & =\sum\limits_{i,i'\in I_n^{\leq p}}f_i^Hf_{i'}^H \prod\limits_{j\in [n]}\I[i_j=i'_j]\\
        & =\sum\limits_{i\in I_n^{\leq p}} (f_i^H)^2.
    \end{align*}
\end{proof}

Then, we have the following lemma that lower bounds the variance of a polynomial with some structure. Given the following lemma, we can apply the anti-concentration results in the proof of Theorem \ref{thm:main-theorem-random-feature-with-perturbation} and \ref{thm:deterministic}.

\begin{lemma}[Variance]\label{lem:variance}
Let $f(x)=f(x_1,x_2\cdots x_d)$ be a homogeneous multivariate polynomial of degree $p$, then there is a symmetric tensor $M\in X_d^p$ that $f(x)=\langle M,x^{\otimes p}\rangle$. For all $x_0\in \R^d$, when $x\sim \cN(0,1)^d$,
\begin{equation*}
    \Var_x[f(x_0+x)]\geq \Vert M\Vert_{\text{rv}}^2
\end{equation*}
\end{lemma}

\begin{proof} We can view $f(x_0+x)$ as a polynomial with respect to $x$ and let $f^M_i$ and $f^H_i$ be the coefficients of its expansion in the monomial basis and Hermite polynomial basis respectively (with variable $x$). It's clear to see that $(f^M_i|i\in I_n^p)$ is the reduced vectorized form of $M$. From the Proposition \ref{pro1} and \ref{pro2}, we have
    \begin{align*}
        \Var[f(x_0+x)]=&\E[f(x_0+x)^2]-\E[f(x_0+x)]^2\\
        =&\sum_{i\in I_n^{\leq p}\backslash 0^n}(f^H_i)^2\\
        \ge& \sum_{i\in I_n^{p}}(f^H_i)^2\geq\sum_{i\in I_n^{p}} (f^M_i)^2\\
        =&\Vert M\Vert_{\text{rv}}^2.
    \end{align*}
\end{proof}

We also need a variance bound for two sets of random variables

\begin{lemma}\label{lem:variance2}
Let $f(x)=f(x_1,x_2\cdots x_d)$ be a homogeneous multivariate polynomial of degree $2p$, then there is a symmetric tensor $M\in X_n^p$ that $f(x)=\langle M,x^{\otimes 2p}\rangle$. For all $u_0, v_0 \in \R^d$, when $u,v\sim \cN(0,I_d)$, we have
\begin{equation*}
    \Var_{u,v}[\langle M, (u_0+u)^{\otimes p}\otimes(v_0+v)^{\otimes p}\rangle]\geq \frac{1}{(2p)!}\Vert M\Vert_{\text{rv}}^2
\end{equation*}
\end{lemma}

%\rong{There is probably a better bound (I suspect $1/{2p \choose p}$ should be enough, but let's go with this for now}

\begin{proof} The proof is similar to Lemma~\ref{lem:variance}. We can view $\langle M, (u_0+u)^{\otimes p}\otimes(v_0+v)^{\otimes p}\rangle$ as a degree-$2p$ polynomial $g$ over $2d$ variables $(u,v)$. Therefore by Lemma~\ref{lem:variance} the variance would be at least the rv-norm of $g$. Note that every element (monomial in the expansion) in $M$ corresponds to at least one element in $g$, and the ratio of coefficient in the correspnding rv-basis is bounded by $(2p)!$, therefore $\|g\|_{rv} \ge \frac{1}{(2p)!} \|M\|_{rv}$, and the lemma follows from Lemma~\ref{lem:variance}.
\end{proof}

\subsection{Proof of Theorem \ref{thm:main-theorem-random-feature-with-perturbation}}\label{subsec:proof-random-feature-perturb}
%\haoyu{Finish the first pass through this subsection. Change $K$ into $k$ and fix the hyperlink problems for label and ref. Missing: Main Lemma and its corresponding proof.}

    In this section, we give the formal proof of Theorem \ref{thm:main-theorem-random-feature-with-perturbation}. First recall the setting of Theorem \ref{thm:main-theorem-random-feature-with-perturbation}: we add a small independent Gaussian perturbation $\tilde{x}\sim \cN(0,v)^d$ on each sample $x$, and denote $\bar{x}=x+\tilde{x}$. The output of the first layer is $\{z_j\}$ where $z_j(i) = (r_i^\top \bar{x}_j)^p$. 
    
    Our goal is to prove that $\{z_j\}$'s satisfy the conditions required by Theorem~\ref{thm:main-theorem-twolayer}, in particular, the matrix $Z = [z_1^{\otimes 2}, ..., z_n^{\otimes 2}]$ has full column rank and a bound on smallest singular value. To do that, note that if we let  $\bar{X}=[\bar{x_1}^{\otimes 2p},\bar{x_2}^{\otimes 2p}\cdots \bar{x_n}^{\otimes 2p}]$ be the order-2p perturbed data matrix, and $Q$ be a matrix whose $i$-th row is equal to $r_i^{\otimes p}$, then we can write $Z = (Q\otimes Q) \bar X$. 
    
    We first show an auxiliary lemma which helps us to bound the smallest singular value of the output matrix $(Q\otimes Q)\bar X$, and then we present our proof for Lemma \ref{lem:smallest-singular-perturb}. 
    
    Generally speaking, the proof of Lemma \ref{lem:smallest-singular-perturb} consists of the lower bound of the \emph{Leave-one-out distance} by the anti-concentration property of polynomials and the use of Lemma \ref{lem:loo-distance-and-singular} to bridge the \emph{Leave-one-out distance} and the smallest singular value.
    
    \begin{lemma}\label{lem:projection}
       Let $M$ be a $k$-dimensional subspace of the symmetric subspace of $X_d^p$, and let $\mbox{Proj}_M$ be the projection into $M$. For any $x\in R^d$ with pertubation $\tilde{x}\sim \cN(0,v)^d$, $\bar{x}=x+\tilde{x}$,
       \begin{equation*}
           \Pr\left\{\Vert \mbox{Proj}_M \bar{x}^{\otimes p}\Vert_2 < \left(\frac{k}{(2p)!}\right)^{1/4}v^{\frac{p}{2}}\varepsilon\right\}< O(p)\varepsilon^{1/p}.
       \end{equation*}
    \end{lemma}
    \begin{proof}
        Let $m_1,m_2\cdots m_k\in X_d^p$ be a set of orthonormal (in $T_d^p$ as a Euclidean space) basis that spans $M$, and each $m_i$ is symmetric. Then $\Vert \mbox{Proj}_M \bar{x}^{\otimes p}\Vert_2 = \sqrt{\sum\limits_{i=1}^k \langle m_i,\bar{x}^{\otimes p}\rangle^2}$. Let $g(x)=\sum\limits_{i=1}^k \langle m_i,x^{\otimes p}\rangle^2=\langle \sum\limits_{i=1}^k m_i^{\otimes 2}, x^{\otimes 2p}\rangle$, then $g(x)$ is a homogeneous polynomial with order $2p$. For any initial value $x$, if $\bar{x} = x+\tilde{x}$, then $\frac{1}{\sqrt{v}} \bar{x} = \frac{1}{\sqrt{v}} x + \frac{1}{\sqrt{v}}\tilde{x}$ is a vector where the random part $\frac{1}{\sqrt{v}}\tilde{x}\sim \cN(0,1)^n$. Therefore by Lemma
        \ref{lem:variance}
        \begin{equation*}
            \Var_x\left[g\left(\frac{1}{\sqrt{v}} \bar{x}\right)\right]\ge \Vert \sum\limits_{i=1}^k m_i^{\otimes 2}\Vert_{\text{rv}}^2\geq \frac{1}{(2p)!} \Vert \sum\limits_{i=1}^k m_i^{\otimes 2}\Vert_{2}^2=\frac{k}{(2p)!} .
        \end{equation*}
        Hence from Proposition \ref{prop:anti-concentration} we know that, when $\hat{x}\sim \cN(0,v I)$,
        \begin{equation*}
            \Pr\left\{\Vert \mbox{Proj}_M \bar{x}^{\otimes p}\Vert_2<\left(\frac{k}{(2p)!}\right)^{1/4}v^{p/2}\varepsilon\right\}=\Pr\left\{\bigg|\sqrt{\frac{(2p)!}{k}}g(\frac{\bar{x}}{\sqrt{v}})\bigg|<\varepsilon^2\right\}\leq O(p)\varepsilon^{1/p}.
        \end{equation*}
    \end{proof}
    
  \begin{lemma}\label{lem:projection2}
       Let $M$ be a $k$-dimensional subspace of the symmetric subspace of $X_d^{2p}$, and let $\mbox{Proj}_M$ be the projection into $M$. For any $x,y\in \R^d$ with pertubation $\tilde{x},\tilde{y}\sim \cN(0,v)^d$, $\bar{x}=x+\tilde{x}$, and $\bar{y}=y+\tilde{y}$, there is
       \begin{equation*}
          \Pr\left\{\Vert \mbox{Proj}_M (\bar{x}^{\otimes p}\otimes \bar{y}^{\otimes p})\Vert_2 < \left(\frac{k}{((4p)!)^2}\right)^{1/4}v^p\varepsilon\right\}< O(p)\varepsilon^{1/{2p}}.
      \end{equation*}
    \end{lemma}
    \begin{proof}
        The proof here is similar to that of Lemma \ref{lem:projection}. Let $m_1,m_2\cdots m_k\in X_d^{2p}$ be a set of orthonormal (in $T_d^{2p}$ as a Euclidean space) basis that spans $M$, and each $m_i$ is symmetric. Then $\Vert \mbox{Proj}_M (\bar{x}^{\otimes p}\otimes \bar{y}^{\otimes p})\Vert_2 = \sqrt{\sum\limits_{i=1}^k \langle m_i,(\bar{x}^{\otimes p}\otimes \bar{y}^{\otimes p})\rangle^2}$. Let $g(x,y)=\sum\limits_{i=1}^k \langle m_i,(x^{\otimes p}\otimes y^{\otimes p})\rangle^2=\langle \sum\limits_{i=1}^k m_i^{\otimes 2}, (x^{\otimes p}\otimes y^{\otimes p}\otimes x^{\otimes p}\otimes y^{\otimes p})\rangle = \langle \sum\limits_{i=1}^k m_i^{(2)}, (x^{\otimes 2p}\otimes y^{\otimes 2p})\rangle$ for some tensor $m_i^{(2)}$,%\runzhe{do we have a better notation for this?}, 
        then $g(x)$ is a homogeneous polynomial with order $4p$. Notice that $\Vert m_i^{(2)}\Vert_2=\Vert m_i^{\otimes 2}\Vert_2$ by a change of coordinate. For any initial value $x$ and $y$, if $\bar{x} = x+\tilde{x}$ and $\bar{y} = y+\tilde{y}$, then $\frac{1}{\sqrt{v}} \bar{x}$ and $\frac{1}{\sqrt{v}} \bar{y}$ are vectors where the random part $\frac{1}{\sqrt{v}}\tilde{x},\frac{1}{\sqrt{v}} \tilde{y}\sim \cN(0,1)^n$. Therefore by Lemma \ref{lem:variance2},
        \begin{align*}
            \Var_{x,y}\left[g\left(\frac{1}{\sqrt{v}} \bar{x},\frac{1}{\sqrt{v}} \bar{y}\right)\right]\ge& \frac{1}{(4p)!}\Vert \sum\limits_{i=1}^k m_i^{(2)}\Vert_{\text{rv}}^2 \\
            \geq& \frac{1}{((4p)!)^2} \Vert \sum\limits_{i=1}^k m_i^{( 2)}\Vert_{2}^2\\
            =&\frac{1}{((4p)!)^2} \Vert \sum\limits_{i=1}^k m_i^{\otimes 2}\Vert_{2}^2\\
            =&\frac{k}{((4p)!)^2} .
        \end{align*}
        Hence from Proposition \ref{prop:anti-concentration} we know that, when $\hat{x}\sim \cN(0,v I)$,
        \begin{align*}
            &\Pr\left\{\Vert \mbox{Proj}_M (\bar{x}^{\otimes p}\otimes \bar{y}^{\otimes p})\Vert_2<\left(\frac{k}{((4p)!)^2}\right)^{1/4}v^{p}\varepsilon\right\}\\
            =&\Pr\left\{\bigg|\frac{(4p)!}{\sqrt{k}}g\left(\frac{\bar{x}}{\sqrt{v}},\frac{\bar{y}}{\sqrt{v}}\right)\bigg|<\varepsilon^2\right\}\leq O(p)\varepsilon^{1/{2p}}.
        \end{align*}
    \end{proof}

Then we can show Lemma \ref{lem:smallest-singular-perturb} as follows.

\lemsmallestsingluarperturb*

Actually, we show a more formal version which also states the dependency on $p$.

\begin{restatable}[Smallest singular value for $(Q\otimes Q)\bar X$ with pertubation]{lemma}{lemsingularlbwithperturb}\label{lem:singular-lb-withperturb}
     With $Q$ being the $k\times d^p$ matrix defined as $Q=[r_1^{\otimes p}, r_2^{\otimes p}\cdots r_k^{\otimes p}]^T$  $(r_i\sim \cN(\mathbf 0, \text{I}))$, with pertubed $\bar{X}=[\bar{x_1}^{\otimes 2p},\bar{x_2}^{\otimes 2p}\cdots \bar{x_n}^{\otimes 2p}]$ $(\bar{x_i}=x_i+\tilde{x_i})$, and with $Z=(Q\times Q)\bar{X}$, when $\tilde{x_i}$ is drawn from i.i.d. Gaussian Distribution $\cN(\mathbf 0, v\text{I})$, for $2\sqrt{n}\leq k\leq \frac{D^{2p}_d}{D_d^p\binom{2p}{p}}=O_p(d^p)$, with overall probability $\geq 1-O(p\delta)$,
    the smallest singular value
    \begin{equation}
        \sigma_{\min}(Z)\geq\left(\frac{[D_d^{2p}-kD_d^p\binom{2p}{p}][{k+1\choose 2}-n]}{[(4p)!]^3}\right)^{1/4}\frac{v^{p}\delta^{4p}}{n^{2p+1/2}k^{4p}}\label{here1}
    \end{equation}
\end{restatable}
\begin{proof}

First, we show that with high probability, the projection of rows of $Q\otimes Q$ in the space of degree $2p$ symmetric polynomials (in this proof we abuse the notation $\mbox{Proj}_{X^{2p}_d}(Q\otimes Q)$ to denote the matrix with rows being the projection of rows of $Q\otimes Q$ onto the space in question) has rank $k_2 := {k+1 \choose 2}$, and moreover give a bound on $\sigma_{k_2}(\mbox{Proj}_{X^{2p}_d}(Q\otimes Q))$.

We do this by bounding the leave one out distance of the rows of $\mbox{Proj}_{X^{2p}_d}(Q\otimes Q)$, note that we only consider rows $(i,j)$ as $\mbox{Proj}_{X^{2p}_d}(r_i^{\otimes p}\otimes r_j^{\otimes p})$ where $1\le i\le j \le k$ (this is because the $(i,j)$ and $(j,i)$-th row of $\mbox{Proj}_{X^{2p}_d}(Q\otimes Q)$ are clearly equal).

The main difficulty here is that different rows of $\mbox{Proj}_{X^{2p}_d}(Q\otimes Q)$ can be correlated. We solve this problem using a technique similar to \cite{ma2016polynomial}.

For any $1\le i\le j \le k$, fix the randomness for $r_l$ where $l\ne i,j$. Consider the subspace $S_{(i,j)} := \mbox{span}\{\mbox{Proj}_{X^{2p}_d}(r_l^{\otimes p} \otimes x^{\otimes p}), x\in \R^d, l\ne i,j\}$. The dimension of this subspace is bounded by $k\cdot D^p_d \cdot {2p\choose p}$ (as there are ${2p\choose p}$ ways to place $p$ copies of $r_l$ and $p$ copies of $x$). Note that any other row of $\mbox{Proj}_{X^{2p}_d}(Q\otimes Q)$ must be in this subspace.

Now by Lemma~\ref{lem:projection2}, we know that the projection of row $(i,j)$ onto the orthogonal subspace of $S_{(i,j)}$ has norm $\left(\frac{D_{d}^{2p}-kD_d^p{2p\choose p}}{((4p)!)^2}\right)^{1/4}\varepsilon$ with probability $O(p)\epsilon^{1/2p}$. Thus by union bound on all the rows, with probability at least $1-O(p\delta)$, the leave-one-out distance is at least
\begin{equation*}
        l(\mbox{Proj}_{X^{2p}_d}(Q\otimes Q))\geq \left(\frac{D_d^{2p}-kD_d^p{2p\choose p}}{((4p)!)^2}\right)^{1/4}\left(\frac{\delta}{\binom{k+1}{2}}\right)^{2p},
    \end{equation*}
    
    and by Lemma \ref{lem:loo-distance-and-singular} the minimal absolute singular value $\sigma_{\min}(\mbox{Proj}_{X^{2p}_d}(Q\otimes Q))\geq \frac{l\left(\mbox{Proj}_{X^{2p}_d}(Q\otimes Q)\right)}{\sqrt{\binom{k+1}{2}}}$.

    %First, we show that with high probability $Q$ is full row rank, and then the by showing each $\bar{x}_i^{2p}$ having a large projecting component onto the subspace of $col(Q)$ orthogonal to $span\{\bar{x}_j^{2p},i\neq j\}$, we would derive a lower bound for the leave-one-out distance for the rows of $(Q\otimes Q)X$, and by extension its least singular value. Hereinafter we take $\varepsilon=\left(\frac{\delta}{k}\right)^{p}$ and $\tau=\left(\frac{\delta}{n}\right)^{2p}$.
    
    %First, for each row $i$ of $Q$, by our Lemma \ref{lem:projection}, let $R_{-i}$ be the subspace of $X_d^p$ spanned by rows other than $i$. From Lemma \ref{lem:projection} we know that with probability $\geq 1 - O(p)\varepsilon^{1/p}$, the projection of row $i$ onto the orthogonal space of $R_{-i}$ (within $X_d^p$, with dimension $D_d^p-k+1$) has norm $\Vert P_{R_{-i}^{\perp}} r_i^{\otimes p}\Vert\geq \left(\frac{D_d ^p-k+1}{(2p)!}\right)^{1/4}\varepsilon$. Thus by union bound on all the rows, with probability $\geq 1-O(p)k\varepsilon^{1/p}=1-O(p\delta)$, for the leave one out distance $l(Q)$ of the rows of $Q$,
    %\begin{equation*}
    %    l(Q)\geq \left(\frac{D_d^p-k}{(2p)!}\right)^{1/4}\varepsilon,
    %\end{equation*}
    %and by Lemma \ref{lem:loo-distance-and-singular} the minimal absolute singular value $\sigma_{\min}|Q|\geq \frac{l(Q)}{\sqrt{k}}$. From linear algebra
    %\begin{equation*}
     %   \sigma_{\min}|Q\otimes Q|=(\sigma_{\min}|Q|)^2\geq
     %   \left(\frac{D_d^p-k}{(2p)!}\right)^{1/2}\frac{\varepsilon^2}{k}.
    %\end{equation*}
    
    Next, let $V(Q\otimes Q)$ be the rowspace of $\mbox{Proj}_{X^{2p}_d}(Q\otimes Q)$, which as we just showed has dimension ${k+1\choose 2}$. We wish to show that the projections of columns of $X$ in $V(Q\otimes Q)$ have a large leave-one-out distance, and thus $(Q\otimes Q)X$ has a large minimal singular value.
    
    Actually for each $i$, the subspace (which for simplicity will be denoted as $V_{-i}(Q\otimes Q)$) of $V(Q\otimes Q)$ orthogonal to $span\{\bar{x}_j^{\otimes 2p}|j\neq i\}$ has dimension ${k+1\choose 2}-n+1$ almost surely, and therefore by Lemma \ref{lem:projection} and union bound, with probability $1-O(p)\tau^{1/2p}n=1-O(p\delta)$, for all $i$,
    \begin{equation*}
        \Vert P_{V_{-i}(Q\otimes Q)} (x_i^{\otimes 2p})\Vert_2 = \E\left[\Vert P_{V_{-i}(Q\otimes Q)} (x_i^{\otimes 2p})\Vert_2\Big|\{\bar{x}_j|j\neq i\}\right]\geq \left(\frac{{k+1\choose 2}-n}{(4p)!}\right)^{1/4}v^{p}\tau,
    \end{equation*}
    thus with probability $1-O(p\delta)$, for any vector $c\in R^n$ with $\Vert c\Vert_2=1$, let $i^*=argmax_{i}|c_i|$, $|c_{i^*}|\geq \frac{1}{\sqrt{n}}$, and
    \begin{equation*}
    \begin{tabular}{rl}
        $\Vert(Q\otimes Q)\hat{X}c\Vert_2$ & $\geq \sigma_{\min}(\mbox{Proj}_{X^{2p}_d}Q\otimes Q)|c_{i^*}| \Vert \mbox{Proj}_{V(Q\otimes Q)} \hat{X}\frac{c}{|c_{i^*}|}\Vert_2$ \\
        & $\geq \frac{\sigma_{\min}(\mbox{Proj}_{X^{2p}_d}Q\otimes Q)}{\sqrt{n}}\Vert \mbox{Proj}_{V_{-{i^*}}(Q\otimes Q)} (x_{i^*}^{\otimes 2p})\Vert_2$ \\
        & $\geq \left(\frac{[D_d^{2p}-kD_d^p\binom{2p}{p}][{k+1\choose 2}-n]}{[(4p)!]^3}\right)^{1/4}\frac{v^{p}\delta^{4p}}{n^{2p+1/2}k^{4p}}$ \\
    \end{tabular}
    \end{equation*}
    
    And therefore we will get Lemma \ref{lem:singular-lb-withperturb}.
\end{proof}

%\rong{There is a small problem in the proof of the lemma above. This is because we need to care about the smallest singular value of $Q\otimes Q$ with rows projected to the space of symmetric tensors $T^{2p}_d$ (as otherwise we cannot use Lemma~\ref{lem:projection}). However here we only proved the smallest singular value of $Q\otimes Q$ is good, without the projection. This is also why in the original proof it seems to be OK as long as $n < k^2$, but this is clearly not going to be correct as in the two-layer setting we do need $n < {k+1\choose 2}$.
%Usually in these cases we can use decoupling technique (the naive way would be to break $Q$ into $2p$ parts, and take one from each part, then everything should be OK. There might be a way to do the decoupling where one only needs to break $Q$ into $2$ parts though. I will think about this tomorrow. The dependencies are not going to change much.}

A minor requirement of on $z_j$'s is that they all have bounded norm. This is much easier to prove:

%Then as shown in the proof sketch of Theorem \ref{thm:main-theorem-random-feature-with-perturbation} in Section \ref{sec:proof-sketch-3-layer-perturb}, we also need to show that with high probability, the norm of $Q\bar x^{\otimes p}$ is upper-bounded. The formal result is shown in the next lemma.

\begin{restatable}[Norm upper bound for $Q\bar x^{\otimes p}$]{lemma}{lemnormupperboundperturb}\label{lem:norm-upperbound-perturb}
    Suppose that $||x_j||_2 \le B$ for all $j\in [n]$ and $\bar x_j = x_j + \tilde x_j$ where $\tilde x_j\sim \cN(\mathbf 0, v\text{I})$. Same as the previous notation, $Q = [r_1^{\otimes p},\dots,r_k^{\otimes p}]^T\in\R^{k\times d^{p}}$. Then with probability at least $1-\frac{\delta}{\sqrt{2\pi\ln((k+n)d\delta^{-1/2})}(k+n)d}$, for all $i\in [n]$, we have
    \[||Q\bar x_i^{\otimes p}||_2 \le \sqrt{k}\left(2(B + 2\sqrt{vd\ln ((k+n)d\delta^{-1/2})})\sqrt{d\ln ((k+n)d\delta^{-1/2})}\right)^p.\]
\end{restatable}

\begin{proof}
    First we have, for a standard normal random variable $N\sim\cN(0,1)$, we have
    \[\Pr\{|N| \ge x\} \le \frac{\sqrt{2}}{\sqrt{\pi}x}e^{-\frac{x^2}{2}}.\]
    Then, apply the union bound, we have with probability at least $1-\frac{\delta}{\sqrt{2\pi\ln((k+n)d\delta^{-1/2})}(k+n)d}$, for all $l\in [k],i\in d, j\in[n], \ell\in d, \delta<1$, we have
    \[|(r_l)_i| \le 2\sqrt{\ln ((k+n)d\delta^{-1/2})}, |(\tilde x_j)_{\ell}| \le 2\sqrt{v\ln ((k+n)d\delta^{-1/2})}.\]
    Then for all $j\in [n]$, we have
    \[||\bar x||_2 \le ||x||_2 + ||\tilde x||_2 \le B + 2\sqrt{vd\ln ((k+n)d\delta^{-1/2})}.\]
    If for all $i\in [d], l\in [k], |(r_j)_i| < 2\sqrt{\ln ((k+n)d\delta^{-1/2})}$, then for any $\bar{x}$ such that 
    \[||\bar{x}|| \le B + 2\sqrt{vd\ln ((k+n)d\delta^{-1/2})}\]
    and any $l\in [k]$, we have
    \begin{align*}
        |\left((r_l)^{\otimes p}\right)^T\bar{x}^{\otimes p}|
        =& |(r_l^T\bar{x})^p|\\
        \le& (||r_l|| \cdot ||\bar{x}||)^p\\
        \le& \left(2(B + 2\sqrt{vd\ln ((k+n)d\delta^{-1/2})})\sqrt{d\ln ((k+n)d\delta^{-1/2})}\right)^p.
    \end{align*}
    Then we have
    \[||Q\bar{x}^{\otimes p}||_2 \le \sqrt{k}\left(2(B + 2\sqrt{vd\ln ((k+n)d\delta^{-1/2})})\sqrt{d\ln ((k+n)d\delta^{-1/2})}\right)^p.\]
\end{proof}

Then combined with the previous lemmas which lower bound the smallest singular value(Lemma \ref{lem:singular-lb-withperturb}) and upper bound the norm(Lemma \ref{lem:norm-upperbound-perturb}) of the outputs of the random feature layer and Theorem \ref{thm:main-theorem-twolayer}, we have the following Theorem \ref{thm:main-theorem-random-feature-with-perturbation}.

\thmrandomfeaturewithperturb*

\begin{proof}
    From the above lemmas, we know that with respective probability $1-o(1)\delta$, after the random featuring, the following happens:
    \begin{enumerate}
        \item $\sigma_{\min}((Q\otimes Q)\bar{X})\geq\left(\frac{[D_d^{2p}-kD_d^p\binom{2p}{p}][{k+1\choose 2}-n]}{[(4p)!]^3}\right)^{1/4}\frac{v^{p}\delta^{4p}}{p^{4p}n^{2p+1/2}k^{4p}}$
        \item $\Vert Q\bar{x}_j^{\otimes p}\Vert_2 \le \sqrt{k}\left(2(B + 2\sqrt{vd\ln ((k+n)d\delta^{-1/2})})\sqrt{d\ln ((k+n)d\delta^{-1/2})}\right)^p$ for all $j\in[n]$.
    \end{enumerate}
    Thereby considering the PGD algorithm on $W$, since the random featuring outputs $[(r_i^T \bar{x}_j)^p]=Q[\bar{x}_j^{\otimes p}]$ has $[(r_i^T \bar{x}_j)^{2p}]=(Q\otimes Q)\bar{X}$, from Theorem \ref{thm:main-theorem-twolayer}, given the singular value condition and norm condition above we obtain the result in the theorem.
\end{proof}

\subsection{Proof of Theorem \ref{thm:deterministic}}\label{subsec:proof-deterministic}
    In this section, we show the proof of Theorem \ref{thm:deterministic}. In the setting of Theorem \ref{thm:deterministic}, we do not add perturbation onto the samples, and the only randomness is the randomness of parameters in the random feature layer.

    Recall that $Q\in\R^{k\times d^p}$ is defined as $Q=[r_1^{\otimes p}, r_2^{\otimes p}\cdots r_k^{\otimes p}]^T$. We show that: when $r_i$ is sampled from i.i.d. Normal distribution $\cN(0,1)^d$ and $k$ is large enough, with high probability $Q$ is robustly full column rank. Let $N_{\varepsilon}$ and $N_{\sigma}$ be respectively an $\varepsilon$-net and a $\sigma$-net of $\bar{X}_d^p$ with size $Z_{\varepsilon}$ and $Z_{\sigma}$.
    
    The following lemmas(Lemma \ref{prop1_L5}, \ref{prop2_L5} and \ref{prop3_L5}) apply the standard $\varepsilon$-net argument and lead to the smallest singular value of matrix $Q$(Lemma \ref{lem:leastSVQ}). Then we will derive the smallest singular value for the matrix $(Q\otimes Q)X$(Lemma \ref{lem:singular-lb-noperturb}).
    
    Note that unlike the $Q$ matrix in the previous section, in this section the $Q$ matrix is going to have more rows than columns, so it has full column rank (restricted to the symmetry of $Q$). The $Q$ matrix in the previous section has full row rank. This is why we could not use the same approach to bound the smallest singular value for $Q$.

    \begin{lemma}\label{prop1_L5}
        For some constant $C$, with probability at least $1-Z_{\varepsilon}\left(Cp\eta^{1/p}\right)^k$, for all $c\in N_\varepsilon$, we have 
        \[\Vert Qc\Vert_2^2 \geq \frac{\eta^2}{p!}.\]
    \end{lemma}
    \begin{proof}
        For any $c\in \bar{X}^p_d$, by Lemma \ref{lem:rv-metric}, $\Vert c\Vert_{\text{rv}}\geq \frac{1}{\sqrt{p!}}$.
        Let $f(r)=c^T r^{\otimes p}$, then $f$ is a polynomial of degree $p$ with respect to $r$, and therefore by Lemma \ref{lem:variance}, 
        \[\mathop{\text{Var}}\limits_{r\sim \cN(0,1)^d}[f(r)]\geq \Vert c\Vert_{\text{rv}}^2\geq\frac{1}{p!}.\]
        Thus by Proposition \ref{prop:anti-concentration},
        \begin{equation*}
            \Pr\limits_{r\sim \cN(0,1)^d}\left\{\big|f(r)\big|< \frac{\eta}{\sqrt{p!}}\right\}\leq O(p)\eta^{1/p}.
        \end{equation*}
        Therefore, as $\Vert Qc\Vert_2^2=\sum\limits_{i=1}^K f(r_i)^2$,
        \begin{align*}
            \Pr\limits_{r_1,r_2\cdots r_K\sim  \cN(0,1)^d}\left\{\Vert Qc\Vert_2^2<\frac{\eta^2}{p!}\right\}\leq & \Pr\limits_{r_1,r_2\cdots r_K\sim  \cN(0,1)^d}\left\{\forall r_i: |f(r_i)|<\frac{\eta}{\sqrt{p!}}\right\}\\
            \leq& \left(O(p)\eta^{1/p}\right)^k.
        \end{align*}
        Therefore for some constant $C$, for each $c\in \bar{X}^p_d$, with probability  at most $\left(Cp\eta^{1/p}\right)^k$ there is $\Vert Qc \Vert_2^2<\frac{\eta^2}{p!}$. Thus by union bound this happens for all $c\in N_\varepsilon$ with probability at most $\leq Z_\varepsilon\left(Cp\eta^{1/p}\right)^k$, and thereby the proof is completed.
    \end{proof}
    
    \begin{lemma}\label{prop2_L5}
        For $\tau>0$, with probability $1-O\left((Z_{\sigma}\left(\frac{\sqrt{k}}{\tau}\right)^{1/p}ke^{-\frac{1}{2}\left(\frac{\tau}{\sqrt{k}}\right)^{2/p}}\right)$,
        for each $c\in N_\sigma$, $\Vert Qc\Vert_2\leq \tau$.
    \end{lemma}
    \begin{proof}
        For any $c\in \bar{X}^p_d$, 
        \begin{align*}
            \Pr\limits_{Q}\left\{\Vert Qc\Vert_2^2> \tau^2\right\}\leq& \Pr\limits_{r_1,r_2\cdots r_k\sim  \cN(0,1)^d}\left\{\exists i: |c^T r_i^{\otimes p}|>\frac{\tau}{\sqrt{k}}\right\}\\
            \leq& k\Pr\limits_{r\sim \cN(0,1)^d}\left\{|c^T r^{\otimes p}|>\frac{\tau}{\sqrt{k}}\right\}.
        \end{align*}
        Furthermore, 
        \begin{align*}
            \Pr\limits_{r\sim \cN(0,1)^d}\left\{|c^T r^{\otimes p}|>\frac{\tau}{\sqrt{k}}\right\} \leq& \Pr\limits_{r\sim \cN(0,1)^d}\left\{\Vert c\Vert_2 \Vert r\Vert_2^p>\frac{\tau}{\sqrt{k}}\right\}\\
            =&\Pr\limits_{r\sim \cN(0,1)^d}\left\{\Vert r\Vert_2>\left(\frac{\tau}{\sqrt{k}}\right)^{1/p}\right\}\\
            \leq& O\left(\left(\frac{\sqrt{k}}{\tau}\right)^{1/p}e^{-\frac{1}{2}\left(\frac{\tau}{\sqrt{k}}\right)^{2/p}}\right)
        \end{align*}
        
        Therefore for the $\sigma$-net $N_{\sigma}$, with a union bound we know with probability at least 
        \[1-O\left((Z_{\sigma}\left(\frac{\sqrt{k}}{\tau}\right)^{1/p}ke^{-\frac{1}{2}\left(\frac{\tau}{\sqrt{k}}\right)^{2/p}}\right),
        \]
        for all $c\in N_{\sigma}$, $\Vert Qc\Vert_2^2\leq\tau^2$.
    \end{proof}
    
    \begin{lemma}\label{prop3_L5}
        For $\sigma<1$, $\tau>0$, with probability at least $ 1-O\left(Z_{\sigma}\left(\frac{\sqrt{k}}{\tau}\right)^{1/p}ke^{-\frac{1}{2}\left(\frac{\tau}{\sqrt{k}}\right)^{2/p}}\right)$, we have
        for each $c\in \bar{X}_d^p$, $\Vert Qc\Vert_2\leq \frac{\tau}{1-\sigma}$.
    \end{lemma}
    \begin{proof}
        We first show that give $N_\sigma$, for each $c\in\bar{X}^p_d$, we can find $c_1,c_2,c_3\cdots\in N_\sigma$ and $a_1,a_2,a_3\cdots\in\R$ such that
        \begin{equation*}
            c=\sum\limits_{i\geq 1}a_ic_i,
        \end{equation*}
        and that $a_1=1$, $0\leq a_i\leq \sigma a_{i-1}$ ($i\geq 2$). Thus $a_i\leq \sigma^{i-1}$.
        
        In fact, we can construct the sequence by induction. Let $I: \bar{X}_d^p\to N_{\sigma}$ that 
        \[I(x)=\mathop{\text{argmin}}\limits_{y\in N_{\sigma}}\Vert y-x\Vert_2.\]
        We take $c_1=I(c)$, $a_1=1$, and recursively 
        \[a_i=\bigg\Vert c-\sum\limits_{j=1}^{i-1}a_jc_j\bigg\Vert_2,\quad c_i=I\left(\frac{c-\sum\limits_{j=1}^{i-1}a_jc_j}{a_i}\right).\]
        By definition, for any $c\in \bar{X}_d^p$, $\Vert c-I(c)\Vert_2\leq\sigma$, and therefore 
        \[\bigg\Vert\frac{c-\sum_{j=1}^{i-1}a_jc_j}{a_i}-c_i\bigg\Vert_2\leq\sigma,\]
        which shows that $0\leq a_{i+1}=\Vert c-\sum\limits_{j=1}^{i}a_jc_j\Vert_2\leq\sigma a_i$, and by induction $a_i\leq \sigma^{i-1}$.
        
        We know from Lemma \ref{prop2_L5} that with probability at least $ 1-O\left(Z_{\sigma}\left(\frac{\sqrt{k}}{\tau}\right)^{1/p}ke^{-\frac{1}{2}\left(\frac{\tau}{\sqrt{k}}\right)^{2/p}}\right)$, for all $c_i\in N_\sigma$, $\Vert Qc_i\Vert_2\leq \tau$, and therefore
        \begin{equation*}
            \Vert Qc\Vert_2\leq\sum\limits_{i\geq 1}a_i\Vert Qc_i\Vert_2 \leq \sum\limits_{i\geq 1}\sigma^{i-1}\tau = \frac{\tau}{1-\sigma}.
        \end{equation*}
    \end{proof}
    
\begin{lemma}[least singular value of $Q$]\label{lem:leastSVQ}
    If $Q$ is the $k\times d^p$ matrix defined as $Q=[r_1^{\otimes p}, r_2^{\otimes p}\cdots r_k^{\otimes p}]^T$ with $r_i$ drawn i.i.d. from Gaussian Distribution $\cN(0,\text{I})$, then there exists constant $G_0>0$ that for $k=\alpha pD_d^p$ ($\alpha>1$), with probability at least $ 1-o(1)\delta$, the rows of $Q$ will span $X_d^p$, and for all $c\in\bar{X}_d^p$, 
    \begin{equation*}
    \Vert Qc\Vert_2\geq \Omega\left(\frac{\delta^{\left(\frac{1}{(\alpha-1)D_d^p}\right)}}{\left(p^p\sqrt{p!}\right)^{\frac{\alpha}{\alpha-1}}\left(k(G_0 p\ln pD_d^p)^p)\right)^{\frac{1}{2(\alpha-1)}}}\right)=\Omega_p\left(\frac{\delta^{\left(\frac{1}{(\alpha-1)D_d^p}\right)}}{k^{\frac{p+1}{2(\alpha-1)}}}\right),
    \end{equation*}
    where $\Omega_p$ is the big-$\Omega$ notation that treats $p$ as a constant.
\end{lemma}
\begin{proof}
    We show that with high probability, for all $c\in \bar{X}^p_d$, $\Vert Qc\Vert^2_2 = \sum\limits_{i=1}^k \left([r_i^{\otimes p}]^T c\right)^2$ is large. To do this we will adopt an $\varepsilon$-net argument over all possible $c$. 
    
    First, we take the parameters
    \[\sigma=\frac{1}{10},\quad \tau=\sqrt{k \left(2\log\frac{Z_\sigma k}{\delta}\right)^p},\quad\text{and}\quad \varepsilon=c_0 \frac{\delta^{\left(\frac{1}{(\alpha-1)D_d^p}\right)}}{\left(\tau p^p\sqrt{p!}\right)^{\frac{\alpha}{\alpha-1}}},\]
    for small constant $c_0$ such that $c_0
    C^pD^{\frac{1}{(\alpha-1)D}} \ll 1$, and $\eta=\frac{20}{9}\varepsilon\tau\sqrt{p!}$. From Lemma \ref{prop1_L5} and \ref{prop3_L5}, we know that with probability at least
    \begin{align*}
    &1-Z_\varepsilon\left(cp\eta^{1/p}\right)^{k}-O\left(Z_{\sigma}\left(\frac{\sqrt{k}}{\tau}\right)^{1/p}ke^{-\frac{1}{2}\left(\frac{\tau}{\sqrt{k}}\right)^{2/p}}\right)\\
    =& 1-O\left(c_0^{(\alpha-1)D_d^p}2^{D_d^p}C^kD\delta\right)-O\left(\frac{\delta}{\sqrt{2\log\frac{Z_\sigma k}{\delta}}}\right)\\
    =&1-o(1)\delta,
   \end{align*}
   the following holds true:
   \begin{enumerate}
       \item $\forall c_i\in N_\varepsilon$, $\Vert Qc_i\Vert_2\geq\frac{\eta}{\sqrt{p!}}$;
       \item $\forall c\in\bar{X}_d^p$, $\Vert Qc\Vert_2\leq \frac{\tau}{1-\sigma}=\frac{\eta}{2\varepsilon\sqrt{p!}}$.
   \end{enumerate}
    Therefore for any $c\in \bar{x}_d^p$, let $i^*=\mathop{\text{argmin}}\limits_{i:c_i\in N_{\varepsilon}}\Vert c-c_i\Vert_2$, we know
    \begin{align*}
        \Vert Qc\Vert_2\geq& \Vert Qc_i\Vert_2 - \Vert Q(c-c_i)\Vert_2\\
        \geq& \frac{\eta}{\sqrt{p!}}- \Vert c-c_i\Vert_2\bigg\Vert Q\frac{c-c_i}{\Vert c-c_i\Vert_2}\bigg\Vert_2\\
        \geq& \frac{\eta}{\sqrt{p!}}- \varepsilon\frac{\eta}{2\varepsilon\sqrt{p!}}\\
        =&\frac{\eta}{2\sqrt{p!}},
    \end{align*}
    and by definition we know that $\lambda_{\min}(Q)\geq \frac{\eta}{2\sqrt{p!}}$. By lemma \ref{lem:eps-net}, with $\log Z_\sigma=O(p\ln pD_d^p)\leq G_0 p\ln pD_d^p$ for some constant $G_0$, this gives us the lemma.
\end{proof}

\begin{restatable}[Smallest singular value for $(Q\otimes Q)X$ without pertubation]{lemma}{lemsingularlbnoperturb}\label{lem:singular-lb-noperturb}
      With $Q$ being the $k\times d^p$ matrix defined as $Q=[r_1^{\otimes p}, r_2^{\otimes p}\cdots r_k^{\otimes p}]^T$, $X$ being the $d^{2p}\times n$ matrix defined as $X = [x_1^{\otimes 2p},\dots,x_n^{\otimes 2p}]\in \R^{d^{2p} \times n}$, and $Z=(Q\otimes Q)X$, for $k=\alpha pD_d^p$ ($\alpha>1$), when $r_i$ are randomly drawn from i.i.d. Guassian distribution $\cN(\mathbf 0, \text{I})$, there exists constant $G_0>0$ such that with probability $\geq 1-o(1)\delta$, the smallest singular value of $Z$ satisfies \begin{equation}\sigma_{\min}(Z)\geq \Omega\left(\frac{\delta^{\left(\frac{2}{(\alpha-1)D_d^p}\right)}\sigma_{\min}(X)}{\left(p^p\sqrt{p!}\right)^{\frac{2\alpha}{\alpha-1}}\left[k(G_0p\ln pD_d^p)^p)\right]^{\frac{1}{(\alpha-1)}}}\right)=\Omega_p\left(\frac{\delta^{\left(\frac{2}{(\alpha-1)D_d^p}\right)}}{k^{\frac{p+1}{(\alpha-1)}}}\right)\sigma_{\min}(X)
    \end{equation}
    (where $\Omega_p$ is the big-$\Omega$ notation that treats $p$ as a constant). Furthermore, for $k=\Omega(p^2 D^p_d)$, with high probability $1-\delta$, $\sigma_{\min}(Z)\geq \Omega(\frac{\sigma_{\min}(X)}{k})$ (if $\delta$ is not exponentially small).
\end{restatable}
\begin{proof}
    From Lemma \ref{lem:leastSVQ}, with probability $\geq 1-o(1)\delta$, for all $c\in\bar{X}_d^p$, 
    \[\Vert Qc\Vert_2\geq\Delta= \Omega\left(\frac{\delta^{\left(\frac{1}{(\alpha-1)D_d^p}\right)}}{\left(p^p\sqrt{p!}\right)^{\frac{\alpha}{\alpha-1}}\left(k(G_0 p\ln pD_d^p)^p)\right)^{\frac{1}{2(\alpha-1)}}}\right).\]
    Then, from linear algebra, we know for all $s\in\bar{X}_d^p\otimes \bar{X}_d^p$, 
    $\Vert(Q\otimes Q)s\Vert_2\geq \Delta^2$. As $\bar{X}_d^{2p}\subset\bar{X}_d^p\otimes \bar{X}_d^p$,
    \begin{equation*}
    \begin{matrix}
        \sigma_{\min}{(Q\otimes Q)X}=\inf\limits_{u\in R^n, \Vert u\Vert_2=1}\Vert (Q\otimes Q)Xu\Vert_2\\
        =\inf\limits_{u\in R^n, \Vert u\Vert_2=1}\Vert (Q\otimes Q)\frac{Xu}{\Vert Xu\Vert_2}\Vert_2\Vert Xu\Vert_2\geq \Delta^2\inf\limits_{u\in R^n, \Vert u\Vert_2=1}\Vert Xu\Vert_2=\Delta^2 \sigma_{\min}(X),
        \end{matrix}
    \end{equation*}
    which gives us this lemma \ref{lem:singular-lb-noperturb}.
\end{proof}

Besides the lower bound for the smallest singular value, we also need the following lemma to show that with high probability, the norm is upper bounded.

\begin{restatable}[Norm upper bound for $Qx^{\otimes p}$]{lemma}{lemnormupperbound}\label{lem:norm-upperbound}
    Suppose that $||x_i||_2 \le B$ for all $i\in [n]$, and $Q = [r_1^{\otimes p},\dots,r_k^{\otimes p}]^T\in\R^{k\times d^{p}}$. Then with probability at least $1-\frac{\delta}{\sqrt{2\pi\ln(kd\delta^{-1/2})}kd}$, for all $i\in [n]$, we have
    \[||Qx_i^{\otimes p}||_2 \le \sqrt{k}\left(2B\sqrt{d\ln (kd\delta^{-1/2})}\right)^p.\]
\end{restatable}

\begin{proof}
    First we have, for a standard normal random variable $N\sim\cN(0,1)$, we have
    \[\Pr\{|N| \ge x\} \le \frac{\sqrt{2}}{\sqrt{\pi}x}e^{-\frac{x^2}{2}}.\]
    Then, apply the union bound, we have
    \begin{align*}
        \Pr\left\{\exists i\in [d], j\in [k], |(r_j)_i| \ge 2\sqrt{\ln (kd\delta^{-1/2})}\right\} \le& kd\frac{\sqrt{2}}{\sqrt{\pi}2\sqrt{\ln (kd\delta^{-1/2})}}\exp{(-2\ln(kd\delta^{-1/2}))}\\ 
        =& \frac{\delta}{\sqrt{2\pi\ln(kd\delta^{-1/2})}kd}.
    \end{align*}
    If for all $i\in [d], j\in [k], |(r_j)_i| < 2\sqrt{\ln (kd)}$, then for any $x$ such that $||x|| \le B$ and any $k_0\in [k]$, we have
    \begin{align*}
        |\left((r_{k_0})^{\otimes p}\right)^Tx^{\otimes p}| =&|(r_{k_0}^Tx)^p|\\
        \le& (||r_{k_0}|| \cdot ||x||)^p\\
        \le& (2B\sqrt{d\ln (kd\delta^{-1/2})})^p.
    \end{align*}
    Then we have
    \[||Qx^{\otimes p}||_2 \le \sqrt{k}\left(2B\sqrt{d\ln (kd\delta^{-1/2})}\right)^p.\]
\end{proof}

Then, combining the previous lemmas and Theorem \ref{thm:main-theorem-twolayer}, we have the following Theorem \ref{thm:deterministic}.

%\thmrandomfeaturenoperturb*
\thmdeterministic*

\begin{proof}
    From the above lemmas, we know that with respective probability $1-o(1)\delta$, after the random featuring, the following happens:
    \begin{enumerate}
        \item There exists constant $G_0$ that $\sigma_{\min}((Q\otimes Q)X)\geq\frac{\delta^{\left(\frac{2}{(\alpha-1)D_d^p}\right)}\sigma_{\min}(X)}{\left(p^p\sqrt{p!}\right)^{\frac{2\alpha}{\alpha-1}}\left[k(G_0p\ln pD_d^p)^p)\right]^{\frac{1}{(\alpha-1)}}}$
        \item $\Vert Qx_j^{\otimes p}\Vert_2 \le \sqrt{k}(2B\sqrt{d\ln (kd\delta^{-1/2})})^p$ for all $j\in[n]$.
    \end{enumerate}
    Thereby considering the PGD algorithm on $W$, since the random featuring outputs $[(r_i^T x_j)^p]=Q[x_j^{\otimes p}]$ has $[(r_i^T x_j)^2p]=(Q\otimes Q)X$, from Theorem \ref{thm:main-theorem-twolayer}, given the singular value condition and norm condition above we obtain the result in the theorem.
\end{proof}

\end{document}